\PassOptionsToPackage{table,xcdraw}{xcolor}
\documentclass{article}
\newif\ifsup\suptrue
\usepackage{xcolor}
\definecolor{dkblue}{cmyk}{1,.54,.04,.19} 
\usepackage{hyperref}
\hypersetup{
    bookmarks=true,         % show bookmarks bar?
    unicode=false,          % non-Latin characters in AcrobatÕs bookmarks
    pdftoolbar=true,        % show AcrobatÕs toolbar?
    pdfmenubar=true,        % show AcrobatÕs menu?
    pdffitwindow=false,     % window fit to page when opened
    pdfstartview={FitH},    % fits the width of the page to the window
    pdftitle={A full classification for adversarial partial monitoring},    % title
    pdfauthor={Lattimore and Szepesvari},     % author
    pdfsubject={partial monitoring},   % subject of the document
    pdfcreator={pdflatex},   % creator of the document
    pdfproducer={Producer}, % producer of the document
    pdfkeywords={partial monitoring} {bandits} {statistics} {machine learning}, % list of keywords
    pdfnewwindow=true,      % links in new window
    colorlinks=true,        % false: boxed links; true: colored links
    linkcolor=dkblue,       % color of internal links
    citecolor=dkblue,       % color of links to bibliography
    filecolor=dkblue,       % color of file links
    urlcolor=dkblue,        % color of external links
}

\usepackage{amsmath}
\usepackage{amssymb}
\usepackage{natbib}
\usepackage{amsthm}
\usepackage[nonatbib,preprint]{nips_2018}

\newtheorem{theorem}{Theorem}
\newtheorem{lemma}{Lemma}

\theoremstyle{definition}

\newtheorem{remark}{Remark}
\newtheorem{exhibit}{Exhibit}
\theoremstyle{remark}
\usepackage{fullpage}

\title{Cleaning up the neighborhood: A full classification for adversarial partial monitoring}
\author{Tor Lattimore \\ DeepMind \\ \And Csaba Szepesv\'ari \\ DeepMind}

\usepackage[capitalize]{cleveref}
\crefname{exhibit}{Exhibit}{Exhibits}
\usepackage{times}
\usepackage{wrapfig}
\usepackage[bf]{caption}

\usepackage{dsfont}
\usepackage[]{algorithm}
\usepackage{algpseudocode}
\usepackage{tikz}
\usetikzlibrary{calc,patterns,arrows,shapes.arrows,intersections}
\usetikzlibrary{decorations}
\usepackage[textsize=tiny,disable]{todonotes}
\newcommand{\todoc}[2][]{\todo[size=\scriptsize,color=blue!20!white,#1]{Csaba: #2}}

\usepackage{enumitem}

\newcommand{\Prob}[1]{\mathbb{P}\left(#1\right)}
\newcommand{\set}[1]{\left\{ #1\right\}}

\newcommand{\one}[1]{\mathds{1}\left\{#1\right\}}

\newcommand{\sone}[1]{\mathds{1}}
\newcommand{\sind}[1]{\mathds{1}_{#1}}
\newcommand{\img}{\operatorname{img}}
\newcommand{\ri}{\operatorname{ri}}
\newcommand{\ones}{\mathbf{1}}
\newcommand{\argmax}{\operatornamewithlimits{arg\,max}}
\newcommand{\argmin}{\operatornamewithlimits{arg\,min}}

\newcommand{\norm}[1]{\left\Vert #1 \right\Vert}
\newcommand{\KL}{\operatorname{KL}}

\newcommand{\E}{\mathbb E}

\newcommand{\ip}[1]{\langle #1 \rangle}
\newcommand{\bbP}{\mathbb P}
\newcommand{\Kloc}{K_{\text{loc}}}
\newcommand{\Vloc}{V_{\text{loc}}}

\newcommand{\cA}{\mathcal A}
\newcommand{\cF}{\mathcal F}
\newcommand{\cL}{\mathcal L}
\newcommand{\cP}{\mathcal P}
\newcommand{\cD}{\mathcal D}
\newcommand{\cG}{\mathcal G}

\newcommand{\cH}{\mathcal H}
\newcommand{\cN}{\mathcal N}
\newcommand{\cV}{\mathcal V}
\newcommand{\cE}{\mathcal E}

\newcommand{\R}{\mathbb R}
\newcommand{\N}{\mathbb N}

\newcommand\numberthis{\addtocounter{equation}{1}\tag{\theequation}}

\let\epsilon\varepsilon
\let\varepsilon\epsilon

\begin{document}

\maketitle

\begin{abstract}
Partial monitoring is a generalization of the well-known multi-armed bandit framework where the loss is not directly observed by the learner.
We complete the classification of finite adversarial partial monitoring to include \textit{all} games, solving an open problem posed by \cite{BFPRS14}. 
Along the way we simplify and improve existing algorithms and correct errors in previous analyses.
Our second contribution is a new algorithm for the class of games studied by \cite{Bar13} where we prove upper and lower regret bounds that shed more light on
the dependence of the regret on the game structure.
\end{abstract}

\section{Introduction}

Partial monitoring is a generalization of the bandit framework that relaxes
the relationship between the feedback and the loss, which
makes the framework applicable to a wider range of practical problems such as spam filtering and product testing. Equally importantly, it
offers a rich and elegant framework to study the exploration-exploitation dilemma beyond bandits \citep{Rus99}.

We consider the finite adversarial version of the problem where a learner and adversary interact over $n$ rounds. At the start of the game the
adversary secretly chooses a sequence of $n$ outcomes from a finite set. In each round the learner chooses one of finitely many actions
and receives a feedback that depends on its action and the choice of the adversary for that round. The loss is also determined by the action/outcome pair, but
is not directly observed by the learner. Although the learner does not know the choices of the adversary, the feedback/loss functions are known
in advance and the learner must use this infer a good policy.
The learner's goal is to minimize the regret, which is the difference between the total loss suffered and the loss that would have been suffered by playing
the best single action given knowledge of the adversaries choices.

\iffalse
The study of partial monitoring games
started with the work \citeauthor{Rus99} in \citeyear{Rus99}, who studied approachability
under a less harsh definition of regret (that allows one to 
make distinctions between hard game instances) 
and who allowed random feedback. \todoc{Mertens et al.'s paper?} 
\fi

The study of partial monitoring games started with the work by \cite{Rus99} where the definition of regret differed slightly from
what is used here and the results have an asymptotic flavor. These results have been strengthened in an interesting line of work by \cite{MS03,Per11b,MPS14}, the last of which gives 
non-asymptotic rates for this more general definition of regret that unfortunately do not reduce to the optimal rate in our setting.
The regret we consider was first considered by \citet{PS01}, who showed that a variant
of exponential weights achieves $\smash{O(n^{3/4})}$ regret in nontrivial games.
This was improved to $\smash{O(n^{2/3})}$ by \citet{CBLuSt06}, who also showed that in general this result is not improvable, but that there
exist many types of game for which the regret is $\smash{O(n^{1/2})}$. They posed the question of classifying finite adversarial partial monitoring games in terms
of the achievable minimax regret. An effort started around $2010$ to achieve this goal, which eventually led to the paper by
\citet{BFPRS14} who made significant progress towards solving this problem. In particular,
they gave an almost complete characterization of partial monitoring games by identifying four regimes: trivial, easy, hard and hopeless 
games. The characterization, however, left out the set of games with actions that are only optimal on low-dimensional subspaces of the adversaries choices. 
Although these actions are never uniquely optimal, they can be informative and until now it was not known how to use these actions when balancing exploration and exploitation.
Games in this tricky regime have been called `degenerate', but there is no particular reason to believe these games should not appear in practice.
This problem is understood in the stochastic variant of partial monitoring where the adversary chooses the outcomes independently at random \citep{ABPS13}, but a complete understanding
of the adversarial setup has remained elusive. 

\paragraph{Contributions} 
\begin{itemize}[leftmargin=*]
\item 
We develop an improved version of \textsc{NeighbourhoodWatch} by \cite{FR12} that correctly deals with degenerate games and
completes the classification for \textit{all} finite partial monitoring games, closing an open question posed by \cite{BFPRS14}.\footnote{Historical note: \cite{FR12} claim a
modification of their argument would handle degenerate games but give no details. The followup paper explicitly mentions the difficulties and poses the open problem \citep[Remark 4 and \S8]{BFPRS14}.}
Another benefit is that \cite{FR12} and \cite{BFPRS14} inadvertently exchanged an expectation and maximum during the localisation
argument of their analysis. A correction is presumably possible, but this would add another level of complexity to an already intricate proof.
Our algorithm also enjoys a regret guarantee that holds with high probability.
\item \cite{Bar13} introduced a class of partial monitoring games and suggested a complicated algorithm with improved regret relative to \textsc{NeighbourhoodWatch}.
We propose a novel algorithm and prove that for these games its regret satisfies $O(F\sqrt{n \Kloc \log(K)})$,
where $K$ is the number of actions and $F$ is the number of feedback symbols. The quantity $\Kloc$ depends on the game and satisfies $\Kloc \leq K$. 
This bound improves on the result of \cite{Bar13} in several ways: \textit{(a)} we eliminate the dependence on arbitrarily large
game-dependent constants, \textit{(b)} the new algorithm is simpler, \textit{(c)} our bound is better by logarithmic factors of the horizon and \textit{(d)} the analysis by \citeauthor{Bar13} mistakenly 
combines bounds that hold in expectation in `local games' into a bound for the whole game as if they were high probability bounds. We expect this could be corrected by modifying the algorithm and analysis, but the resulting
algorithm would be even more complicated and the regret would not improve.
\item We prove a variety of lower bounds. First correcting a minor error in the proof by \cite{BFPRS14} and second showing the linear dependence 
on the number of feedbacks is unavoidable in general.
\item The new algorithms and analysis simplify existing results, which think is a contribution in its own right and we hope encourages more 
research into this fascinating topic with many open questions. 
\end{itemize}

\paragraph{Problem setup}
Given a natural number $n$ let $[n] = \{1,2,\dots,n\}$.
We use $\ip{x,y}$ to denote the usual inner product in Euclidean space.
The $d$-simplex is $\cP_d = \{x \in [0,1]^{d+1} : \norm{x}_1 = 1\}$, where for $p\ge 1$, $\norm{x}_p$ is the $p$-norm of $x$.
The relative interior of $\cP_d$ is $\ri(\cP_d) = \{x \in (0,1)^{d+1} : \norm{x}_1 = 1\}$.
The dimension of a set $A\subset \smash{\R^{d+1}}$ is the dimension of its affine hull.
For any set $A$ the indicator function is $\sind{A}(\cdot)$ and for function $f:A \to \R$ the supremum norm of $f$ is $\norm{f}_\infty = \sup_{a \in A} |f(a)|$. 
A partial monitoring problem $G = (\cL,\Phi)$ is a game between a learner and an adversary over $n$ rounds
and is specified by a \textit{loss matrix} $\cL \in [0,1]^{K \times E}$ and a \textit{feedback matrix}  $\Phi \in [F]^{K \times E}$
for natural numbers $E, F$ and $K$.
At the beginning of the game the learner is given $\cL$ and $\Phi$ and the adversary secretly chooses a sequence of \textit{outcomes} $i_{1:n} = (i_1,\ldots,i_n)$ where $i_t \in [E]$ for each $t \in [n]$. 
In each round $t$ the learner chooses an action $A_t \in [K]$ and observes feedback $\Phi_t = \Phi_{A_ti_t}$. 
The loss incurred by playing action $a$ in round $t$ is $y_{ta} = \cL_{ai_t}$. 
In contrast to bandit and full information problems the loss in partial monitoring is \textit{not} observed 
by the learner, even for the action played.
\begin{wrapfigure}[7]{r}{3.7cm}
\vspace{-0.4cm}
\small
\renewcommand{\arraystretch}{1.4}
\hspace{-0.3cm}
\begin{tabular}{|
>{\columncolor[HTML]{EFEFEF}}l |l|}
\hline
\cellcolor[HTML]{FFFFFF}\rlap{Game type}
\phantom{Feedback ($\Phi$)}
 & \cellcolor[HTML]{EFEFEF}$R^*_n(G)$ \\ \hline
Trivial  & 0  \\ \hline
Easy     & $\tilde \Theta(n^{1/2})$ \\ \hline
Hard     & $\Theta(n^{2/3})$ \\ \hline
Hopeless & $\Omega(n)$ \\ \hline
\end{tabular}
\end{wrapfigure}
A policy $\pi$ is a map from sequences of action/observation pairs to a distribution over the action-set $[K]$.
The performance of a policy $\pi$ is measured by its \textit{regret}, $R_n(\pi, i_{1:n}) = \max_{a \in [K]} \sum_{t=1}^n (y_{tA_t} - y_{ta})$.
When the outcome sequence and policy are fixed we abbreviate $R_n = R_n(\pi, i_{1:n})$.
The minimax expected regret associated with partial monitoring game $G$ is the worst-case expected regret of the best policy.
$R^*_n(G) = \inf_{\pi} \max_{i_{1:n}} \E[R_n(\pi, i_{1:n})]$
where the inf is taken over all policies, the max over all outcome sequences of length $n$ and the expectation with respect to the randomness in the actions.
We let $\cF_t = \sigma(A_1,A_2,\ldots,A_t)$ be the $\sigma$-algebra generated by the information available after round $t$ and abbreviate $\E_t[\cdot] = \E[\cdot|\cF_t]$.
A core question in partial monitoring is to understand how $\cL$ and $\Phi$ affect the growth of $R^*_n(G)$ in terms of the horizon.
The main theorem of \cite{BFPRS14} shows that for all `nondegenerate' games the minimax regret falls into one of four categories as illustrated in the table.
The colloquial meaning of the adjective degenerate suggests that only nondegenerate games are interesting, but this is not the case.
The term is used in a technical sense (to be clarified soon) referring to a subclass of games that we have no reason to believe should be less important than the nondegenerate ones.

\paragraph{Preliminaries}
To illustrate some of the difficulties of partial monitoring relative to bandits we formalize a simplistic version of the spam filtering problem.

\paragraph{Example 1}
%%%%%%%%%%%%%%%%%%%%%%%%%%%%%%%%%%%%%%%%%%%%%%%%%%%%%%
%%%%%%%%%%%%%%%%%%%%%%%%%%%%%%%%%%%%%%%%%%%%%%%%%%%%%%
\begin{wrapfigure}[7]{r}{5.5cm}
\vspace*{-0.29in}
\centering
\small
\begin{tabular}{|
>{\columncolor[HTML]{EFEFEF}}l |l|l|}
\hline
\cellcolor[HTML]{FFFFFF}\rlap{Loss ($\cL$)}
\phantom{Feedback ($\Phi$)}
 & \cellcolor[HTML]{EFEFEF}Spam & \cellcolor[HTML]{EFEFEF}Not spam \\ \hline
Spam                               & 0                            & 1                                \\ \hline
Not spam                           & 1                            & 0                                \\ \hline
Don't know                             & c                            & c                                \\ \hline
\end{tabular}

\medskip
\begin{tabular}{|
>{\columncolor[HTML]{EFEFEF}}l |l|l|}
\hline
\cellcolor[HTML]{FFFFFF}
Feedback ($\Phi$)
 & \cellcolor[HTML]{EFEFEF}Spam & \cellcolor[HTML]{EFEFEF}Not spam \\ \hline
Spam                               & 1                            & 1                                \\ \hline
Not spam                           & 1                            & 1                                \\ \hline
Don't know                              & 1                            & 2                                \\ \hline
\end{tabular}

\end{wrapfigure}
%%%%%%%%%%%%%%%%%%%%%%%%%%%%%%%%%%%%%%%%%%%%%%%%%%%%%%
%%%%%%%%%%%%%%%%%%%%%%%%%%%%%%%%%%%%%%%%%%%%%%%%%%%%%%
Let $c \geq 0$ and define partial monitoring game $G = (\cL, \Phi)$ by
\begin{align*}
\cL = \begin{pmatrix}
0 & 1 \\
1 & 0 \\
c & c
\end{pmatrix}\,, \qquad
\Phi = \begin{pmatrix}
1 & 1 \\
1 & 1 \\
1 & 2
\end{pmatrix}\,.
\end{align*}
The idea is also illustrated in the tables on the right.
Rows correspond to actions of the learner and columns to outcomes selected by the adversary.
The learner has three actions in this game corresponding to `spam', `not spam' and `don't know' while the
adversary chooses between `spam' and `not spam'. The learner suffers a loss of $1$ if it guesses incorrectly. Alternatively the learner can
say they don't know in which case they suffer a loss of $c$ and observe some meaningful feedback. 
The minimax regret for this game depends on the price of information. 
If $c > 1/2$, then the minimax regret is $\Theta(n^{2/3})$. On the other hand, if $c \in (0, 1/2]$ the minimax regret is $\tilde \Theta(n^{1/2})$ where $\tilde \Theta(\cdot)$ indicates
growth up to logarithmic factors. Finally, when $c = 0$ a policy can suffer no regret by playing just the third action.

\paragraph{Example 2}
\begin{wrapfigure}[3]{r}{3.2cm}
\vspace{-0.8cm}
\begin{align*}
\cL = \begin{pmatrix}
0 & 1 \\
1 & 0
\end{pmatrix}\,\,\,
\Phi = \begin{pmatrix}
1 & 1 \\
1 & 1
\end{pmatrix}
\end{align*}
\end{wrapfigure}
The game on the right is hopeless because the learner cannot gain information about her loss and
the adversary can always force the expected regret to be $\Omega(n)$.

\paragraph{Cell decomposition}
In order to understand what makes a partial monitoring game hard, easy or hopeless, it helps to introduce a linear structure.
Let $u_t = e_{i_t} \in \cP_{E-1}$ be the standard basis vector that is nonzero in the coordinate of the outcome $i_t$ chosen by the adversary in round $t$.
For action $a$ let $\ell_a \in [0,1]^E$ be the $a$th row of matrix $\cL$. 
The \textit{cell} $C_a$ of action $a$ is the subset of $\cP_{E-1}$ on which action $a$ is optimal:
$C_a = \{u \in \cP_{E-1} : \max_{b \in [K]} \ip{\ell_a - \ell_b, u} = 0\}$.
Action $a$ is optimal in hindsight if and only if $\frac{1}{n} \sum_{t=1}^n u_t \in C_a$.
Each nonempty $C_a$ is a polytope and 
the collection $\{C_a : a \in [K]\}$ is called the \textit{cell decomposition} of $G$.
An action is called \emph{dominated} if it is never optimal: $C_a = \emptyset$. 
We define the \emph{dimension} of nondominated action $a$ to be the dimension of $C_a$, which ranges between $0$ and $E-1$. 
Nondominated actions with dimension less than $E-1$ are called \textit{degenerate} while actions with dimension $E-1$ are called \textit{Pareto optimal}.
A partial monitoring game is \emph{degenerate} if it has at least one degenerate action.
For each $u \in \cP_{E-1}$ let $\smash{a^*_u \in \argmin_a \ip{\ell_a, u}}$ and $\smash{a^*_t \in \argmin_a \sum_{s=1}^t \ip{\ell_a, u_s}}$, which means that $a^*_u$ is an 
optimal action if the adversary is playing $u$ on average and $a^*_t$ is the optimal action in hindsight when the adversary plays the sequence $(u_1,\ldots,u_t)$.
Without loss of generality we assume that $a^*_u$ and $a^*_t$ are nondegenerate.
A pair of nondegenerate actions $a, b$ are \emph{neighbors} if $C_a \cap C_b$ has dimension $E-2$. They are \emph{weak neighbors} if $C_a \cap C_b \neq \emptyset$.
Actions $a$ and $b$ are called \emph{duplicates} if $\ell_a = \ell_b$. 
We let $\cN_a$ be the set of actions consisting of $a$ and its neighbors (but \textit{not} the duplicates of $a$). For any pair of 
neighbors $(a,b)$ let $\cN_{ab} = \set{c \in [K] : C_a \cap C_b \subseteq C_c}$. Although $a$ is not a neighbor of itself we define $\cN_{aa} = \emptyset$.

\begin{lemma}[\citealt{BFPRS14}, Lem.\ 11]\label{lem:pm:degenerate}
Let $a$ and $b$ be neighbors.
Then for all $d \in \cN_{ab}$ there exists a unique $\alpha \in [0,1]$ such that $\ell_d = \alpha \ell_a + (1 - \alpha) \ell_b$.
\end{lemma}

A corollary is that for $d \in \cN_{ab}$ and if $\alpha$ from the lemma lies in $(0,1)$, then $C_d = C_a \cap C_b$. 
Degenerate and dominated actions can never be uniquely optimal in hindsight, but they can provide information to the learner that proves the difference between
a hard and hopeless game (or easy and hard). 
This is also true for duplicate actions, which have the same loss, but not necessarily the same feedback.

\paragraph{Observability}
The neighborhood structure determines which actions can be uniquely optimal and when.
This is only half of the story. The other half is the relationship between the feedback and loss matrices that defines the difficulty of identifying the optimal action.
A natural first attempt towards designing an algorithm would be to construct an unbiased estimator of $y_{ta}$ for each Pareto optimal action $a$. 
\ifsup
A moments thought produces easy games where this is impossible (Exhibit~\ref{exmpl:pm:no-estimates} in Appendix~\ref{sec:pm:gallery}). 
\else
A moments thought produces easy games where this is impossible (see the supplementary material for an example).
\fi
A more fruitful idea is to estimate the loss differences $y_{ta} - y_{tb}$ for Pareto optimal actions $a$ and $b$,
which is sufficient (and essentially necessary) to discover the optimal action.
Suppose in round $t$ the learner has chosen to sample $A_t \sim P_t$ where $P_t \in \ri(\cP_{K-1})$.
A conditionally unbiased estimator of $y_{ta} - y_{tb}$ is a function $g:[K] \times [F] \to \R$ such that $\E_{t-1}[g(A_t, \Phi_t)] = \sum_a P_{ta} g(a, \Phi_{ai_t}) = y_{ta} - y_{tb}$.
Whether or not such an estimator exists and its structure determines the difficulty of a partial monitoring game.
A pair of actions $(a,b)$ are called \emph{globally observable} if there exists a function $v: [K] \times [F] \to \R$ such that
\begin{align}
\sum_{c=1}^K v(c, \Phi_{ci}) = \ell_{ai} - \ell_{bi} \qquad \text{for all } i \in [E]\,.
\label{eq:pm:v}
\end{align}
They are \emph{locally observable} if in addition to the above $a$ and $b$ are neighbors and
$v(c, f) = 0$ whenever $c \notin \cN_{ab}$. Finally, they are \emph{pairwise observable} if $v(c, f) = 0$ whenever $c \notin \{a, b\}$. 
If the learner is sampling action $A_t$ from distribution $P_t \in \ri(\cP_{K-1})$, then the existence of a function satisfying \cref{eq:pm:v}
means that $v(A_t, \Phi_t) / P_{tA_t}$
is an unbiased estimator of $\ip{\ell_a - \ell_b, u_t} = y_{ta} - y_{tb}$. 
A game $G$ is called \emph{globally/locally observable} if all pairs of neighbors are globally/locally observable.
A game is called \emph{point-locally observable} if all pairs of weak neighbors are pairwise observable.
The cell decomposition and observability structure for the spam game is described in detail in \ifsup Exhibit~\ref{exmpl:pm:spam}. \else the supplementary material. \fi
Note that in globally observable games it is easy to see that any pair of Pareto optimal actions are globally observable, not just the neighbors.

\section{Classification theorem}
The following theorem classifies partial monitoring games into four categories depending on the observability structure.

\begin{theorem}\label{thm:pm:classification}
The minimax regret of partial monitoring game $G = (\cL, \Phi)$ satisfies
\begin{align*}
R_n^*(G) = \begin{cases}
0, & \text{if } G \text{ has no pairs of neighboring actions }; \\
\tilde \Theta(\sqrt{n}), & \text{if } G \text{ is locally observable and has neighboring actions }; \\
\Theta(n^{2/3}), & \text{if } G \text{ is globally observable, but not locally observable }; \\
\Omega(n), & \text{otherwise}\,.
\end{cases}
\end{align*}
\end{theorem}

The theorem follows by proving upper and lower bounds for each class of games. Most of the pieces already exist in the literature. 
The upper bound for globally observable games is by \cite{CBLuSt06}.
The upper bound for games with no pairs of neighboring actions is trivial, since in this case there exists an action $a$ with $C_a = \cP_{E-1}$ and playing this
action alone ensures zero regret. The lower bound for easy games is by \citet[\S6]{ABPS13} and for hard games by \citet[\S4]{BFPRS14}. 
All that remains is to prove an upper bound for locally observable games with at least one pair of neighboring actions.

%%%%%%%%%%%%%%%%%%%%%%%%%%%%%%%%%%%%%%%%%%%%%%%%%%%%%%%%%%%%%%%%%%%%%%
% ALGORITHM FOR LOCALLY OBSERVABLE GAMES
%%%%%%%%%%%%%%%%%%%%%%%%%%%%%%%%%%%%%%%%%%%%%%%%%%%%%%%%%%%%%%%%%%%%%%
\section{Algorithm for locally observable games}\label{sec:pm:nw2}

\begin{wrapfigure}[3]{r}{4.5cm}
\centering
\vspace{-2.1cm}
\begin{tikzpicture}[font=\scriptsize,scale=0.8]
\begin{scope}
\clip (0,0) -- (3,0) -- (0,3) -- (0,0);
\draw[draw=none,fill=gray!10!white] (0,0) -- (0,1.5) -- (1.5,1.5) -- (1.5,0) -- (0,0);
\draw[draw=none,fill=gray!30!white] (1.5,0) -- (1.5,1.5) -- (3,0) -- (1.5,0);
\draw[draw=none,fill=gray!30!white] (0,1.5) -- (0,3) -- (1.5,1.5) -- (0,1.5);
\draw[densely dotted,ultra thick] (0,1.5) -- (1.5,1.5);
\end{scope}
\draw[densely dotted,ultra thick] (0,1.5) -- (0,0);
\draw[fill=black] (0,1.5) circle (2pt);
\draw[->] (0,0) -- (3.3,0);
\draw[->] (0,0) -- (0,3.3);
\node at (0.75,0.75) {$C_3$};
\node at (0.375, 1.875) {$C_2$};
\node at (1.875,0.375) {$C_1$};
\node[anchor=west] at (3.4,0) {$u_1$};
\node[anchor=south] at (0,3.4) {$u_2$};
\node[inner sep=0pt] (C4) at (1.5,2.3) {$C_4$};
\draw[thin,-latex,shorten >=3pt] (C4) -- (0.75,1.5);
\node[inner sep=0pt] (C5) at (-1,0.75) {$C_5$};
\draw[thin,-latex,shorten >=3pt] (C5) -- (0, 0.75);
\node[inner sep=0pt] (C6) at (-1,1.5) {$C_6$};
\draw[thin,-latex,shorten >=3pt] (C6) -- (0, 1.5);
\end{tikzpicture}
\end{wrapfigure}
Fix a locally observable game $G = (\cL,\Phi)$ with at least one pair of neighboring actions.
We introduce a policy called \textsc{NeighborhoodWatch2} (\cref{alg:nw2}).

\paragraph{Preprocessing}
The new algorithm always chooses its action $A_t \in \cup_{a,b} \cN_{ab}$ where the union is over pairs of neighboring actions.
For example, in the game with cell decomposition shown in the figure the policy only plays actions $1$, $2$, $3$ and $4$. 
Removing (some) degenerate actions can only increase the minimax regret so from now on
we assume that all actions in $[K]$ are in $\cN_{ab}$ for some neighbors $a$ and $b$.
Let $\cA$ be an arbitrary largest subset of Pareto optimal actions
such that $\cA$ does not contain actions that are duplicates of each other 
and $\cD = [K] \setminus\cA$ be the remaining actions.

\paragraph{Estimating loss differences}
The definition of local observability means that
for each pair of neighboring actions $a,b$ there exists a function $v^{ab} : [K] \times [F] \to \R$ satisfying \cref{eq:pm:v} and with
$v^{ab}(c, f) = 0$ whenever $c \notin \cN_{ab}$. Even though $a$ is not a neighbor of itself, for notational convenience we define $v^{aa}(c, f) = 0$ for all $c$ and $f$.
The policy works for any such of $v^{ab}$, but the analysis suggests minimizing $\smash{V = \max_{a,b} \Vert v^{ab}\Vert_\infty}$ with the maximum over all pairs of neighbors. 

\begin{algorithm}[H]
\begin{algorithmic}[1]
\State \textbf{Input\,\,} $\cL$, $\Phi$, $\eta$, $\gamma$
\For{$t \in 1,\ldots,n$}
\State For $a,k \in [K]$ let 
$\displaystyle Q_{tka} = \sind{\cA}(k) \frac{\sind{\cN_k\cap \cA}(a)\exp\left(-\eta \sum_{s=1}^{t-1} \tilde Z_{ska} \right)}{\sum_{b \in \cN_k\cap \cA} \exp\left(-\eta \sum_{s=1}^{t-1} \tilde Z_{ska}\right)} + \sind{\cD}(k) \frac{\sind{\cA}(a)}{|\cA|}$
\State Find distribution $\tilde P_t$ such that $\tilde P_t^\top = \tilde P_t^\top Q_t$
\State Compute $P_t = (1 - \gamma) \text{\textsc{Redistribute}}(\tilde P_t) + \frac{\gamma}{K} \ones$ and sample $A_t \sim P_t$
\State Compute loss-difference estimators for each $k \in \cA$ and $a \in \cN_k \cap \cA$.
\begin{align}
\hat Z_{tka} = \frac{\tilde P_{tk} v^{ak}(A_t, \Phi_t)}{P_{tA_t}} \quad \text{and} \quad
\beta_{tka} = \eta V^2 \sum_{b \in \cN_{ak}} \frac{\tilde P_{tk}^2}{P_{tb}} \quad \text{and} \quad
\tilde Z_{tka} = \hat Z_{tka} - \beta_{tka} 
\label{eq:pm:estimators}
\end{align}
\EndFor
\Function{Redistribute}{$p$}
\State $q \gets p$
\For{$d \in \cD$}
\State Find $a, b$ such that $d \in \cN_{ab}$ and $\alpha \in [0,1]$ such that $\ell_d = \alpha \ell_a + (1 - \alpha) \ell_b$ \label{line:ab} \quad (\cref{lem:pm:degenerate})
\State $c_a \gets \frac{\alpha q_b}{\alpha q_b + (1 - \alpha) q_a}$ and $c_b \gets 1 - c_a$ and $\rho \gets \frac{1}{2K}\min\set{\frac{p_a}{q_a c_a},\, \frac{p_b}{q_b c_b}}$
\State $q_d \gets \rho c_a q_a + \rho c_b q_b$ and $q_a \gets (1 - \rho c_a) q_a$ and $q_b \gets (1 - \rho c_b) q_b$ 
\EndFor
\State \Return $q$
\EndFunction
\end{algorithmic}
\caption{\textsc{NeighborhoodWatch2}}\label{alg:nw2}
\end{algorithm}

\paragraph{Description}
In each round the algorithm first computes a collection of exponential weights distribution $Q_{tk} \in \cP_{K-1}$, one for each $k \in \cA$.
The distribution $Q_{tk}$ is supported on the $\cN_k \cap \cA$ when $k \in \cA$ and for $k \in \cD$ it is uniform on $\cA$.
These local distributions are then combined into a global distribution $\smash{\tilde P_t}$, which is taken to be the stationary distribution of right-stochastic matrix $Q_t$, which means that
\begin{align}
\tilde P_{ta} = \sum_{k \in \cA} \tilde P_{tk} Q_{tka} \text{ for any } a, k \in \cA\,.
\label{eq:pm:std}
\end{align}
These steps are the same as the original \textsc{NeighborhoodWatch}, which samples its action from $(1 - \gamma) \smash{\tilde P_t} + \gamma \ones / K$.
This does not work when there are degenerate actions because 
$Q_{tkd} = 0$ when $d \in \cD$, which by the above display means that $\smash{P_{td} = \gamma / K}$ for actions $d \in \cD$ and non-adaptive forced exploration
is not sufficient for $O(\sqrt{n})$ regret in partial monitoring. This is the role of the redistribution function, which is analyzed formally in
\ifsup
Appendix~\ref{app:redistribute}.
\else
the supplementary material.
\fi
The final part of the algorithm is to estimate the loss differences for each $k \in \cA$ and $a \in \cN_k \cap \cA$.
Our choice of loss estimators are another departure from the original algorithm, which only updated the estimators for one local game in each round and then
used a complicated aggregation strategy. This is one source of significant simplification in the new algorithm.  

\begin{remark}
The special treatment of degenerate actions using the redistribution function seems like a big hassle. You might wonder why we did not simply include the degenerate actions in the local games and
then play the stationary distribution, possibly with a little exploration. Unfortunately this idea does not work. 
Let $d$ be a degenerate action in $\cN_{ak}$ where $a$ and $k$ are neighbors. Then \cref{lem:pm:degenerate} shows that the loss-difference between $k$ and $d$ can be estimated by
$\smash{\hat Z_{skd} = \alpha \hat Z_{skk} + (1 - \alpha) \hat Z_{ska}}$ with $\alpha$  such that $\ell_d = \alpha \ell_k + (1 - \alpha) \ell_a$.
Intuitively, a degenerate action $d$ in $\cN_{ak}$ is only useful for learning about the loss differences between actions $a$ and $k$, which
suggests the algorithm should not assign much more probability to $d$ than the minimum probability of playing $a$ and $k$. At a technical level the proof does not go through because
the predictable variation of the estimator above is roughly $\Omega(\max(1/P_{tk}, 1/P_{ta}))$ and yet $P_{td}$ can be $\Omega(\max(P_{tk}, P_{ta}))$ and in the analysis
of exponential weights these terms are required to cancel.
\end{remark}

\begin{remark}
The estimators $\tilde Z_{tka}$ are negatively biased by $\beta_{tka}$ in order to prove high probability bounds, which is reminiscent of the Exp3.P algorithm for
finite-armed adversarial bandits \citep{ACFS02}.
As a minor contribution, we generalize their analysis to the case where the loss estimators satisfy certain constraints, rather than taking the specific importance-weighted
form used for adversarial bandits. Choosing $\beta_{tka} = 0$ in the algorithm leads to a bound on the expected regret as we soon show.
\end{remark}

\begin{theorem}\label{thm:pm:easy}
Suppose \cref{alg:nw2} is run on locally observable $G = (\cL, \Phi)$ with parameters $\delta \in (0,1)$ and
$\smash{\eta = \frac{1}{V}\sqrt{\log(K/\delta) / (nK)}}$ and $\gamma = VK \eta$.
Then with probability at least $1 - \delta$ the regret is bounded by $R_n \leq C_G \smash{\sqrt{n \log(e/\delta)})}$, where $C_G$ is a constant that depends on the game $G$, but not the horizon $n$
or confidence level $\delta$.
\end{theorem}

%%%%%%%%%%%%%%%%%%%%%%%%%%%%%%%%%%%%%%%%%%%%%%%%%%%%
% PROOF SKETCH
%%%%%%%%%%%%%%%%%%%%%%%%%%%%%%%%%%%%%%%%%%%%%%%%%%%%
\ifsup 
The complete proof of \cref{thm:pm:easy} given in Appendix~\ref{app:thm:pm:easy}.
Here we prove a bound on the expected regret in the simple case where there are no degenerate actions and $\beta_{tka} = 0$. Although this proof does not highlight one of our main contributions (how to deal with
degenerate actions), it does emphasize the enormous simplification of the new algorithm.
\else
The proof of \cref{thm:pm:easy} is given in the supplementary material. As consolation we prove a bound on the expected regret
in the simple case where there are no degenerate actions and $\beta_{tka} = 0$. Although this proof does not highlight one of our main contributions (how to deal with
degenerate actions), it does emphasize the enormous simplification of the new algorithm.
\fi
The first step is a localization argument to bound the regret in terms of the `local regret' in each neighborhood. We need a simple lemma, which for completeness we prove in
the
\ifsup the appendix. \else the supplementary material. \fi

\begin{lemma}[\citealt{BFPRS14}]\label{lem:pm:external-internal}
There exists a constant $\epsilon_G > 0$ depending only on $G$ such that for 
all pairs of actions $a, \tilde a \in \cA$ and $u \in C_{\tilde a}$ there exists an action $b \in \cN_a \cap \cA$ such that
$\ip{\ell_a - \ell_{\tilde a}, u} \leq \ip{\ell_a - \ell_b, u} / \epsilon_G$.
\end{lemma}

Since there are no degenerate actions, the \textsc{Redistribute} function has no effect and $\smash{P_t = (1 - \gamma) \smash{\tilde P_t} + \gamma \ones / K}$. 
Let $B_1,\ldots,B_n$ be a sequence of random variables with $\smash{B_t \sim \tilde P_t}$ that is conditionally independent of $A_t$ given the observations up to time $t$. Then
by Hoeffding-Azuma's inequality
\begin{align}
R_n = \sum_{t=1}^n \ip{\ell_{A_t} - \ell_{a^*_n}, u_t} 
\leq n\gamma + \sqrt{8 \log(1/\delta)} + \sum_{t=1}^n \ip{\ell_{B_t} - \ell_{a^*_n}, u_t}\,.
\label{eq:pm:sketch-first}
\end{align}
Next we apply \cref{lem:pm:external-internal} to localize the second term,
\begin{align*}
\textrm{(A)}
= \sum_{t=1}^n \ip{\ell_{B_t} - \ell_{a^*_n}, u_t}
= \sum_{a \in [K]} \left\langle \ell_a - \ell_{a^*_n}, {\textstyle\sum_{t: B_t = a} u_t} \right\rangle
\leq \frac{1}{\epsilon_G} \max_{\phi \in \cH} \sum_{t=1}^n \ip{\ell_{B_t} - \ell_{\phi(B_t)}, u_t}\,,
\end{align*}
where $\cH$ is the set of functions $\phi : [K] \to [K]$ with $\phi(a) \in \cN_a$ for all $a$.
Then using Hoeffding-Azuma's inequality and a union bound over all $\phi \in \cH$ shows that with probability at least $1 - \delta$,
\begin{align}
\textrm{(A)}
&\leq \frac{1}{\epsilon_G} \max_{\phi \in \cH} \sum_{t=1}^n \sum_{a \in [K]} \tilde P_{ta} \ip{\ell_a - \ell_{\phi(a)}, u_t} + \sqrt{8n \log\left(\frac{|\cH|}{\delta}\right)} \nonumber \\
&= \frac{1}{\epsilon_G} \max_{\phi \in \cH} \sum_{t=1}^n \sum_{k \in [K]} \tilde P_{tk} \sum_{a \in \cN_k} Q_{tka} \ip{\ell_a - \ell_{\phi(k)}, u_t} + \sqrt{8n \log\left(\frac{|\cH|}{\delta}\right)} \nonumber \\
&= \frac{1}{\epsilon_G} \sum_{k \in [K]} \underbrace{\max_{b \in \cN_k} \sum_{t=1}^n \sum_{a \in \cN_k} Q_{tka} (\tilde P_{tk} y_{ta} - \tilde P_{tk} y_{tb})}_{\text{local regret}} + \sqrt{8n \log\left(\frac{|\cH|}{\delta}\right)}\,,
\label{eq:sketch-global}
\end{align}
where the first equality uses the fact that $\tilde P_t$ is the stationary distribution of $Q_t$ (see (\ref{eq:pm:std})).
The local regret is bounded using the tools from online convex optimization. Of course the losses are never actually observed and must be replaced with the loss difference estimators.
Then it remains to control the variance of these estimators. The `standard' analysis of Exp3 \citep{ACFS95,Ces06} shows that
\begin{align}
\max_{b \in \cN_k} \sum_{t=1}^n Q_{tka} \left(\tilde P_{tk} y_{ta} - \tilde P_{tk} y_{tb}\right)
\leq \frac{\log(K)}{\eta} + \eta \sum_{t=1}^n \sum_{a \in \cN_k} Q_{tka} \hat Z_{tka}^2 \,.
\label{eq:sketch-local}
\end{align}
In order to bound the second term we substitute the definition of $\hat Z_{tka}$, which shows that
\begin{align*}
\sum_{a \in \cN_k} Q_{tka} \hat Z_{tka}^2 
&= \sum_{a \in \cN_k} \frac{\tilde P_{tk}^2 Q_{tka}  v^{ak}(A_t, \Phi_t)^2}{P_{tA_t}^2} 
\leq \frac{\tilde P_{tk}V^2}{P_{tA_t}} \sum_{a \in \cN_k} \frac{\tilde P_{tk}Q_{tka} \sind{\{a,k\}}(A_t)}{P_{tA_t}} 
\leq \frac{2\tilde P_{tk}V^2}{P_{tA_t}}\,.
\end{align*}
where in the first inequality we used the fact that $\Vert v^{ak} \Vert_\infty \leq V$. The second inequality follows by considering two cases. First, if $A_t = k$, then all entries of the sum are
non-zero and $\sum_{a \in \cN_k} \smash{\tilde P_{tk} Q_{tka} = \tilde P_{tk} \leq 2P_{tA_t}}$, which is true by choosing $\gamma \leq 1/2$. 
For the second case $A_t = a$ for $a \in \cN_k$ and $a \neq k$, which means that only one term of the sum is non-zero. 
Then the definition of $\smash{\tilde P_t}$ as the stationary distribution of $Q_t$ means that $\smash{\tilde P_{tk} Q_{tka} \leq \tilde P_{ta} \leq 2 P_{tA_t}}$. 
Combining this with \cref{eq:pm:sketch-first,eq:sketch-global,eq:sketch-local} and a union bound shows that with probability at least $1 - \delta$.
\begin{align*}
R_n
&\leq n \gamma + \frac{1}{\epsilon_G}\left(\frac{K \log(K)}{\eta} + 2\eta V^2 \sum_{t=1}^n \frac{1}{P_{tA_t}} \right) + \sqrt{8n \log(2/\delta)} +\sqrt{8 n \log\left(\frac{2|\cH|}{\delta}\right)}\,. 
\end{align*}
Now $\E[\sum_{t=1}^n P_{tA_t}^{-1}] = nK$, which means that 
\begin{align*}
\E[R_n] 
&\leq n\gamma + \frac{1}{\epsilon_G} \left(\frac{K \log(K)}{\eta} + 2\eta nK V^2\right) + 2\sqrt{8n(1 +  \log(2|\cH|))}
= O\left(\frac{KV}{\epsilon_G} \sqrt{n \log(K)}\right)\,,
\end{align*}
where we first used \cref{lem:dom} below along with naive bounding and the fact that $|\cH| \leq K^K$. The Big-O follows by choosing
$\eta = \smash{\frac{1}{V} \sqrt{\log(K)/n}}$ and $\gamma = \eta KV$.
The choice of $\gamma$ ensures that the loss-difference estimate satisfies $\eta |\smash{\hat Z_{tka}}| \leq \eta V / P_{tA_t} \leq \eta V K / \gamma = 1$ on which the proof of \cref{eq:sketch-local} relies.
We prove \ifsup in Appendix~\ref{sec:pm:v} \else in the supplementary material \fi that for games without degenerate actions the loss-difference estimators can always
be chosen so that $V \leq 1 + F$. 

\begin{lemma}\label{lem:dom}
Suppose $a \geq 0$ and $b \geq 1$ are constants and $X, Y$ are random variables such that $\bbP(X \geq Y + \sqrt{a \log(b/\delta)}) \leq \delta$ for all $\delta \in (0,1)$.
Then $\E[X] \leq \E[Y] + \sqrt{a(1 + \log(b))}$. 
\end{lemma}

\paragraph{Dealing with degenerate actions}
The presence of degenerate actions makes the calculation significantly more fiddly. The first step is to show that the redistribution process guarantees that the expected loss
accumulated by playing $P_t$ rather than $\smash{\tilde P_t}$ is not too great. The localization argument is then repeated and the remaining question is how to control the variance of
the loss difference estimates. The redistribution process guarantees that the degenerate actions have sufficient mass that the variance is at most $O(K)$ larger than what we saw in the
above calculation. The process is complicated slightly by the desire to have a high probability bound.

%%%%%%%%%%%%%%%%%%%%%%%%%%%%%%%%%%%%%%%%%%%%%%%%%%%%
% STRONG LOCALLY OBSERVABLE GAMES
%%%%%%%%%%%%%%%%%%%%%%%%%%%%%%%%%%%%%%%%%%%%%%%%%%%%
\section{Algorithm for point-locally observable games}\label{sec:pm:simple}

The weakened neighbor definition and pairwise observability makes the analysis of point-locally observable games less delicate than locally observable games 
and the results are correspondingly stronger. Perhaps the most striking improvement is that asymptotically the bound does not dependent on arbitrarily large game-dependent constants.
Here we present a simple new algorithm based on \textsc{Exp3} called \textsc{RelExp3} (`Relative Exp3').
The name is derived from the fact that the algorithm does not estimate losses directly, but rather the loss differences relative to an `anchor' arm that varies over time and is the arm
to which the algorithm assigns the largest probability. As we shall see, this reduces the variance of the loss difference estimates.

\paragraph{Preprocessing}
The definition of pairwise observability means that degenerate and dominated actions are not needed to estimate the loss differences. Since removing these actions can only increase the minimax
regret, for the remainder of this section we fix a point-locally observable game $G = (\cL, \Phi)$ for which there are no dominated or degenerate actions.
A \textit{point-local game} is a largest subset of actions $A \subseteq [K]$ with $\bigcap_{a \in A} C_a \neq \emptyset$ (a maximal clique of the graph over actions with edges representing weak neighbors). 
We let $\Kloc$ be the size of the largest point-local game.

\paragraph{Estimation functions}
For each pair of actions $a,b$ let $v^{ab}$ be an estimation function satisfying \cref{eq:pm:v} and furthermore assume that
$v^{aa} = 0$ and $v^{ab}(c, f) = 0$ if $a, b$ are weak neighbors and $c \notin \{a, b\}$.
The existence of these functions is guaranteed by the definition of a point-locally observable game.
Given pair of actions $a,b$ let $S^{ab}$ be the set of actions needed to estimate the loss difference between $a$ and $b$, which is 
$S^{ab} = \{a, b\} \cup \{c \in [K] : \text{exists } f \in [F] \text{ such that }  v^{ab}(c, f) \neq 0\}$.
Our assumptions ensure that $S^{ab} = \{a, b\}$ if $a$ and $b$ are weak neighbors.
Define
$\smash{V^{ab} = \Vert{v^{ab}}\Vert_\infty}$ and $\smash{V = \max_{a,b \in [K]} V^{ab}}$ and $\smash{\Vloc = \max_{a,b : C_a \cap C_b \neq \emptyset} V^{ab}}$.
\ifsup
We show in \cref{sec:pm:v} that $v^{ab}$ can be chosen so that $\Vloc \leq 1 + F$.
\else 
We show in the supplementary material that $v^{ab}$ can be chosen so that $\Vloc \leq 1 + F$.
\fi

\paragraph{Decreasing learning rates}
The algorithm makes use of a sequence of decreasing learning rates $(\eta_t)_{t=1}^\infty$ and exploration parameters $(\alpha_t)_{t=1}^\infty$.
On top of this the algorithm also has a dynamic exploration component that ensures the loss difference estimates are not too large.
The decreasing learning rate is one of the essential innovations that allows us to prove an asymptotic bound that is independent of arbitrarily large game-dependent quantities.
As an added bonus, it also means the algorithm does not require advance knowledge of the horizon. 

\begin{algorithm}[H]
\begin{algorithmic}[1]
\State $\hat L_{0a} = 0$ for all $a \in [K]$
\For{$t = 1,\ldots,n$}
\State  For each $a \in [K]$ let $\displaystyle \tilde P_{ta} = \frac{\exp(-\eta_t \hat L_{t-1,a})}{\sum_{b=1}^K \exp(-\eta_t \hat L_{t-1,b})}$ 
\State  Let $B_t = \argmax_a \tilde P_{ta}$\,\, and \,\, $\displaystyle M_t = \set{a : \tilde P_{ta} \exp\left(\frac{\eta_t V^{aB_t}}{\alpha_t}\right) > \frac{\eta_t}{t}}$ 
\State  Let $\displaystyle S_t = \bigcup_{a \in M_t} S^{aB_t}$ \,\, and\,\,  $\displaystyle \gamma_{ta} = \sind{S_t}(a) \eta_t \max_{a \in M_t} V^{aB_t} + \frac{\alpha_t}{K}$ \,\, and \,\, $P_t = (1 - \norm{\gamma_t}_1) \tilde P_{ta} + \gamma_t$ 
\State   Sample $A_t \sim P_t$ and observe feedback $\Phi_t$ 
\State For each $a \in [K]$ compute estimates $\displaystyle \hat Z_{ta} = \frac{v^{aB_t}(A_t, \Phi_t)}{P_{tA_t}}$ and update $\hat L_{ta} = \hat L_{t-1,a} + \hat Z_{ta}$
\EndFor
\end{algorithmic}
\caption{\textsc{RelExp3}}\label{alg:pm:exp3pm}
\end{algorithm}

\begin{theorem}\label{thm:pm}
Let $G = (\cL, \Phi)$ be point-locally observable, then with appropriately tuned parameters \textsc{RelExp3} satisfies
$\limsup_{n\to\infty} \E[R_n] / \sqrt{n}
\leq 8\sqrt{2\Kloc (1 + F) (2+F)\log(K)}$. Furthermore, the linear dependence on $F$ is \ifsup unavoidable (see Appendix~\ref{sec:pm:lower-more}). \else unavoidable. \fi
\end{theorem}
Note that the constant hidden by the asymptotics \textit{does} depend on arbitrarily large game-dependent
constants. The proof of \cref{thm:pm} \ifsup may be found in the \cref{app:thm:pm}, \else may be found in the supplementary material, \fi
but the general idea is to show the forced exploration ensures for sufficiently large $t$ that the algorithm is almost always playing in a point-local game that contains the optimal
action and at this point the variance of the importance-weighted estimators is well behaved.

\section{Summary and open problems}

We completed the classification of all finite partial monitoring games. Along the way we greatly simplified existing algorithms and analysis and proved that
for a large class of games the asymptotic regret does not depend on arbitrarily large game-dependent constants, which is the first time this has been demonstrated
in the adversarial setting.
There are many fascinating open problems. One of the most interesting is to understand to what extent it is possible to adapt to `easy data'.
For example, globally observable games may have locally observable subgames and one might hope for an algorithm with $O(\sqrt{n})$ regret
if the adversary is playing in this subgame and $O(n^{2/3})$ regret otherwise. Another question is to refine the definition of the regret to differentiate between algorithms in
hopeless games where linear regret is unavoidable, but the coefficient can depend on the algorithm \citep{Rus99}.
Yet another question is to understand to what extent $V$ is a fundamental quantity in the regret for easy games and whether or not the arbitrarily large game-dependent constants
are real for large $n$ as we have shown they are not for point-locally observable games.

\clearpage

\bibliography{all}

\ifsup

\appendix

\section{Redistribution properties}\label{app:redistribute}

Here we collect a number of properties of the \textsc{Redistribute} function in \cref{alg:nw2}. 

\begin{lemma}\label{lem:pm:redistribute}
Assume $\gamma \in [0,1/2]$ and let $u\in \cP_{E-1}$, and  $k,a\in \cA$ arbitrary neighbors. 
Then $P_t \in \cP_{K-1}$ is a probability vector and the following hold: 
\begin{align*}
&\text{\textit{(a)}}\quad P_{ta} \geq \tilde P_{ta} / 4\,; &
&\text{\textit{(b)}}\quad \left|\sum_{a=1}^K (P_{ta} - \tilde P_{ta}) \ip{\ell_a, u}\right| \leq \gamma\,; \\
&\text{\textit{(c)}}\quad P_{tb} \geq \frac{\tilde P_{tk} Q_{tka}}{4K}\, \text{ for any non-duplicate } b \in \cN_{ka}\,;  &
&\text{\textit{(d)}}\quad P_{ta} \geq \gamma / K \,; \\
&\text{\textit{(e)}}\quad P_{td} \geq \frac{\tilde P_{tk}}{4K}\, \text{ for any } d\in [K] \text{ such that } \ell_d = \ell_k\,.
\end{align*}
\end{lemma}

\begin{proof}
First we show that $P_t$ is indeed a probability vector. By assumption $\tilde P_t$ is the stationary distribution, which is a probability distribution.
Let $\bar P_t = \text{\textsc{Redistribute}}(\tilde P_t)$ so that 
\begin{align*}
P_t = (1 - \gamma) \bar P_t + \frac{\gamma}{K} \ones\,,
\end{align*}
which means we need to show that $\bar P_t$ is a probability distribution.
Since $\bar P_t$ is obtained by the iterative procedure given in the \textsc{Redistribute} function it is sufficient to show that the vector $q$ tracked by this algorithm is indeed
a distribution. The claim is that each loop of the \textsc{Redistribute} function does not break this property. The first observation is that the algorithm always moves mass from actions in $\cA$
to actions in $\cD$. All that must be shown is that $\bar P_{ta} \geq 0$ for all $a \in \cA$.
To see this note first that if $a \in \cA$ is one of the choices of the algorithm in Line~\ref{line:ab}, then
$\rho c_a q_a \leq p_a / (2K)$ and so
\begin{align}
\bar P_{ta} \geq \tilde P_{ta} / 2 \qquad \text{for all } a \in \cA \geq 0\,. \label{eq:pm:half}
\end{align}

\noindent Part (a): 
Since $\gamma \leq 1/2$ this follows from \cref{eq:pm:half}. 

\noindent Part (b): 
First we show that $\sum_{a \in [K]} (\bar P_{ta} - \tilde P_{ta}) \ell_a = 0$. It suffices to show that the redistribution in each inner loop of
the algorithm does not change this value, which is true because
\begin{align*}
(c_a q_a + c_b q_b) \ell_d 
&= (c_a q_a + c_b q_b) (\alpha \ell_a + (1 - \alpha) \ell_b) \\
&= \frac{q_a q_b}{\alpha q_b + (1 - \alpha)q_a} (\alpha \ell_a + (1 - \alpha) \ell_b) \\
&= \rho c_a q_a \ell_a + \rho c_b q_b \ell_b\,.
\end{align*}
Then using the definition of $P_t$ we have
\begin{align*}
\left|\sum_{a \in [K]} (P_{ta} - \tilde P_{ta}) \ip{\ell_a, u}\right|
= \left|\sum_{a \in [K]} (P_{ta} - \bar P_{ta}) \ip{\ell_a, u}\right| 
= \gamma \left|\sum_{a \in [K]} \left(\frac{1}{K} - \bar P_{ta}\right) \ip{\ell_a, u}\right|
\leq \gamma\,,
\end{align*}
where we used the assumption that $\ell_a \in [0,1]^E$ for all actions and $u \in \cP_{E-1}$ so that $\ip{\ell_a, u} \in [0,1]$.

\noindent Part~(c): There are three cases: Either $b=k$ or $b=a$ or $b$ is degenerate.
If $b = k$, then the result is immediate from Part (a). 
If $b = a$, then, Part~(a) combined with~\eqref{eq:pm:std} implies that $P_{tb} = P_{ta} \geq \tilde P_{ta} / 4 \geq \tilde P_{tk} Q_{tka} / 4 \geq \tilde P_{tk} Q_{tka} / (4K)$. 
Finally, if $b$ is degenerate, then by the definition of the rebalancing algorithm we have
\begin{align*}
\bar P_{tb} \geq \frac{\min(\tilde P_{tk}, \tilde P_{ta})}{2K} \geq \frac{\min(\tilde P_{tk}, \tilde P_{tk} Q_{tka})}{2K} = \frac{\tilde P_{tk} Q_{tka}}{2K}
\end{align*}
and the result follows from \cref{eq:pm:half}.

\noindent Part~(d): This is trivial from the definition of $P_t$. 

\noindent Part~(e): Let $b \in \cA$ be the Pareto optimal action chosen by the rebalancing algorithm when $d$ is given weight. 
Since $\ell_d = \ell_a$ it follows that $\alpha = 1$ and so $c_a = 1$ and $c_b = 1$, which means that $\bar P_{td} = \tilde P_{ta} / 2$ and
using \cref{eq:pm:half} again yields the result.
\end{proof}

\section{Proof of Theorem~\ref{thm:pm:easy}}\label{app:thm:pm:easy}

We start by proving \cref{lem:pm:external-internal}.

\begin{wrapfigure}[7]{r}{3cm}
\vspace{-0.2cm}
\begin{tikzpicture}[scale=0.8]
\draw[thin] (0,0) -- (3,0) -- (0,3) -- (0,0);
\draw[thin] (0,1.5) -- (1.5,1.5);
\draw[thin] (1.5,0) -- (1.5,1.5);
\draw[fill=black] (2.3,0.3) circle (1pt);
\draw[fill=black] (0.5,2) circle (1pt);
\node (Ca) at (0.3,2.3) {\footnotesize {$C_a$}};
\node (Cta) at (1.8,0.3) {\footnotesize {$C_{\tilde{a}}$}};
\node (Cb) at (0.3,0.3) {\footnotesize {$C_{b}$}};
\node (v) at (1.3,2.5) {$v$};
\node (u) at (3.1,0.8) {$u$};
\node (w) at (1.9,2.2) {$w$};
\draw[->,thin,shorten >=2pt] (v) -- (0.5,2);
\draw[->,thin,shorten >=2pt] (u) -- (2.3,0.3);
\draw[->,thin,shorten >=2pt] (w) -- (1.02,1.5);
\draw[densely dotted,thin,<-,shorten <=2pt,shorten >=2pt] (2.3,0.3) -- (0.5,2);
\draw[fill=black] (1.02,1.5) circle (1pt);
\end{tikzpicture}
\end{wrapfigure}
\textit{Proof of Lemma~\ref{lem:pm:external-internal}.}\, Since $u \in C_{\tilde a}$, $0\le \ip{\ell_a-\ell_{\tilde a},u}$.
The result is trivial if $a, \tilde a$ are neighbors or $\ip{\ell_a - \ell_{\tilde a}, u} = 0$.
From now on assume that $\ip{\ell_a - \ell_{\tilde a}, u} > 0$ and that $a, \tilde a$ are not neighbors.
Let $v$ be the centroid of $C_a$ and consider the line segment connecting $v$ and $u$. Then let $w$ be the first point
on this line segment for where there exists a $b \in \cN_a \cap \cA$ with $w \in C_b$ (see figure).
Note that $w$ is well-defined by the Jordan-Brouwer separation theorem and $b$ is well-defined
because $\cA$ is a maximal duplicate-free subset of the Pareto optimal actions. %
Using twice that $\ip{\ell_a-\ell_b,w}=0$,
we calculate
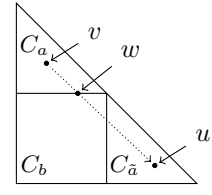
\begin{align}
\ip{\ell_a-\ell_b,u} 
& = \ip{\ell_a-\ell_b,u-w} 
= \frac{\norm{ u - w }_2}{\norm{ v - w }_2} \, \ip{\ell_a -\ell_b,w-v}
=\frac{\norm{ u - w }_2}{\norm{ v - w }_2} \, \ip{\ell_b -\ell_a,v}>0\,,
\label{eq:pm:abuidentity}
\end{align}
where the second equality used that $w\ne v$ is a point of the line segment connecting $v$ and $u$, hence $w-v$ and $u-w$ are parallel and share the same direction
and $\norm{v-w}_2 > 0$.
The last inequality follows because $v$ is the centroid of $C_a$ and $a,b$ are distinct Pareto optimal actions.
Let $v_c$ be the centroid of $C_c$ for any $c \in \cA$. Then,
\begin{align*}
\frac{\ip{\ell_a - \ell_{\tilde a}, u}}{\ip{\ell_a - \ell_b, u}}
&= \frac{\ip{\ell_a - \ell_{\tilde a}, w + u - w}}{\ip{\ell_a - \ell_b, u}} 
\stackrel{\text{\footnotesize (a)}}{\leq} 
\frac{\ip{\ell_a - \ell_b, w} + \ip{\ell_a - \ell_{\tilde a}, u - w}}{\ip{\ell_a - \ell_b, u}} \\
&
\stackrel{\text{\footnotesize (b)}}{=}
 \frac{\ip{\ell_a -\ell_{\tilde a}, u-w}}{\ip{\ell_a - \ell_b, u}} 
\stackrel{\text{\footnotesize (c)}}{=}
 \frac{ \norm{v - w}_2 \ip{\ell_a - \ell_{\tilde a}, u - w}}{ \norm{u - w}_2 \ip{\ell_b - \ell_a, v}} \\
&
\stackrel{\text{\footnotesize (d)}}{\leq}
\frac{\norm{v - w}_2 \norm{\ell_a - \ell_{\tilde a}}_2}{\ip{\ell_b - \ell_a, v}} 
\stackrel{\text{\footnotesize (e)}}{\leq}
\frac{\sqrt{2E}}{\min_{c \in \cA} \min_{d \in \cN_c} \ip{\ell_d - \ell_c, v_c}} 
= \frac{1}{\epsilon_G}\,, 
\end{align*}
where (a)  follows since by \eqref{eq:pm:abuidentity}, $\ip{\ell_a-\ell_b,u}>0$ and also because
$w \in C_b$ implies that $\ip{\ell_a - \ell_{\tilde a}, w} \leq \ip{\ell_a - \ell_b, w}$,
(b) follows since $\ip{\ell_a-\ell_b,w}=0$ (which is used in other steps, too),
(c) uses \eqref{eq:pm:abuidentity}, 
(d) is by Cauchy-Schwartz
and in (e) 
we bounded $\norm{w - v}_2 \leq \sqrt{2}$ and 
used that $\norm{\ell_a - \ell_{\tilde a}}_2 \leq \sqrt{E}$ 
and $\ip{\ell_b - \ell_a, v} = \ip{\ell_b - \ell_a, v_a} \geq \min_{c \in \cA} \min_{d \in \cN_c} \ip{\ell_d - \ell_c, v_c}>0$. 
The final equality serves as the definition of $1/\epsilon_G$. \qed

\begin{lemma}\label{lem:pm:external-internal-2}
Let $\cH$ be the set of functions $\phi:\cA \to \cA$ with $\phi(a) \in \cN_a$ for all $a \in \cA$
and define
$a_n^* = \argmin_{a\in [K]} \sum_{t=1}^n \ip{ \ell_{a}, u_t}$. Then, for any $(B_t)_{1\le t \le n}$ sequence of actions in $\cA$,
\begin{align*}
\sum_{t=1}^n \ip{\ell_{B_t} - \ell_{a^*_n}, u_t} 
\leq \frac{1}{\epsilon_G} \max_{\phi \in \cH} \sum_{t=1}^n \ip{\ell_{B_t} - \ell_{\phi(B_t)}, u_t}\,.
\end{align*}
\end{lemma}

\begin{proof}
With no loss of generality we assume that $a_n^*\in \cA$ 
because $\cA$ is a maximal duplicate-free subset of Pareto optimal actions.
Apply the previous lemma on subsequences of rounds where $B_t = a$ for each $a \in \cA$.
\end{proof}

\begin{lemma}\label{lem:pm:transform}
Let $\delta \in (0,1)$. Then with probability at least $1 - 2\delta$ it holds that
\begin{align*}
R_n \leq \gamma n +  \frac{1}{\epsilon_G}  \sum_{k\in \cA} \max_{b\in \cN_k\cap \cA} \sum_{t=1}^n \tilde P_{tk} \sum_{a \in \cA} Q_{tka} \left(y_{ta} - y_{tb}\right) + \sqrt{8n \log(|\cH|/\delta)}\,.
\end{align*}
\end{lemma}

\begin{proof}
For $t\in [n]$, let $B_t \sim \tilde P_t$.
Define the surrogate regret $R_n' = \sum_{t=1}^n \ip{\ell_{B_t} - \ell_{a^*_n}, u_t}$.
By the definition of $A_t$ and $B_t$ and part (b) of \cref{lem:pm:redistribute} we have $\E_{t-1}[\ip{\ell_{A_t} - \ell_{B_t}, u_t}] \le \gamma$. 
Furthermore, $|\ip{\ell_a - \ell_b, u_t}| \leq 1$ for all $a, b$. 
Therefore, by Hoeffding-Azuma, with probability at least $1 - \delta$, \index{Hoeffding--Azuma}
\begin{align}
R_n \leq R_n' + \gamma n +  \sqrt{2n \log(1/\delta)}\,. 
\label{eq:pm:surrogate}
\end{align}
By \cref{lem:pm:external-internal-2}, 
the surrogate regret is bounded in terms of the local regret:
\begin{align}
R_n' 
&=  \sum_{t=1}^n \ip{\ell_{B_t} - \ell_{a^*_n}, u_t} 
\leq \frac{1}{\epsilon_G} \max_{\phi \in \cH} 
\sum_{t=1}^n \ip{\ell_{B_t} - \ell_{\phi(B_t)}, u_t}\,.
\label{eq:pm:surrogate2}
\end{align}
We prepare to use Hoeffding-Azuma again.
Fix $\phi \in \cH$ arbitrarily. Then,
\begin{align*}
\E_{t-1}\!\left[\ip{\ell_{B_t} - \ell_{\phi(B_t)}, u_t}\right] 
&= \sum_{k \in \cA} \tilde P_{tk} \sum_{a \in \cA} Q_{tka} \ip{\ell_a - \ell_{\phi(k)}, u_t} \numberthis \label{eq:pm:markovkey} \\
&= \sum_{k \in \cA} \tilde P_{tk} \sum_{a \in \cA} Q_{tka} (y_{ta} - y_{t\phi(k)})\,,
\end{align*}
where we used the fact that $\tilde P_{ta} = \sum_k \tilde P_{tk} Q_{tka}$.
Hoeffding-Azuma's inequality now shows that with probability at least $1 - \delta / |\cH|$,
\begin{align}
\sum_{t=1}^n \ip{\ell_{B_t} - \ell_{\phi(B_t)}, u_t}
& \leq
 \sum_{k \in \cA} \sum_{t=1}^n \tilde P_{tk} \sum_{a \in \cA} Q_{tka} (y_{ta} - y_{t\phi(k)}) + \sqrt{2n \log(|\cH|/\delta)}\,. \nonumber
\end{align}
The result is completed via a union bound over all $\phi \in \cH$ and chaining with \cref{eq:pm:surrogate,eq:pm:surrogate2}, and noting that
\begin{align*}
\max_{\phi} \sum_{k \in \cA} \sum_{t=1}^n \tilde P_{tk} \sum_{a \in \cA} Q_{tka} (y_{ta} - y_{t\phi(k)})
&\le \sum_{k \in \cA}\max_{\phi}  \sum_{t=1}^n \tilde P_{tk} \sum_{a \in \cA} Q_{tka} (y_{ta} - y_{t\phi(k)}) \\
&= \sum_{k \in \cA}
\underbrace{\max_{b\in \cN_k \cap \cA}  \sum_{t=1}^n \tilde P_{tk} \sum_{a \in \cA} Q_{tka} (y_{ta} - y_{tb})}_{R_{nk}}\,. \qedhere
\end{align*}
\end{proof}

\begin{proof}[Proof of \cref{thm:pm:easy}]
The proof has two steps. First bounding the local regret $R_{nk}$ for each $k\in \cA$ and then merging the bounds using the previous lemma.

\paragraph{Step 1: Bounding the local regret}
For the remainder of this step we fix $k \in \cA$ and bound the local regret $R_{nk}$. First, we need to massage the local regret into a form in which we can apply 
\cref{thm:pm:exp3p}, which is a generic version of the Exp3.P analysis by \cite{ACFS02}.
Let $Z_{tka} = \tilde P_{tk} (y_{ta} - y_{tk})$
and $\cG_t$ be the $\sigma$-algebra generated by $(A_1,\dots,A_t)$ and $\cG = (\cG_t)_{t=0}^n$ be the 
associated filtration.
A simple rewriting shows that
\begin{align*}
R_{nk}
&= \max_{b \in \cN_k\cap \cA} \sum_{t=1}^n \tilde P_{tk} \sum_{a \in \cA} Q_{tka} \left(y_{ta} - y_{tb}\right) 
 = \max_{b \in \cN_k \cap \cA} \sum_{t=1}^n \sum_{a \in \cA} Q_{tka} \left(Z_{tka} - Z_{tkb}\right)\,.
\end{align*}
In order to apply the result in \cref{thm:pm:exp3p} we need to check the conditions. Since $(P_t)_t$ and $(\tilde P_t)_t$ are $\cG$-predictable it follows that 
$(\beta_t)_t$ and $(Z_t)_t$ are also $\cG$-predictable. 
Similarly, $(\hat Z_t)_t$ is $\cG$-adapted because $(A_t)_t$ and $(\Phi_t)_t$ are $\cG$-adapted. 
It remains to show that assumptions \textit{(a--d)} are satisfied. For \textit{(a)} let $a \in \cN_k \cap \cA$. By \cref{lem:pm:redistribute}.(d) we have $P_{tb} \geq \gamma / K$ 
for all $t$ and $b \in [K]$. Furthermore, $|v^{ak}(A_t, \Phi_t)| \leq V$ so that
$\eta |\hat Z_{tka}| = |\eta \tilde P_{tk} v^{ak}(A_t, \Phi_t)/ P_{tA_t}| \leq \eta VK / \gamma = 1$, where the equality follows from the choice of $\gamma$.
Assumption \textit{(b)} is satisfied in a similar way with
$\eta \beta_{tka} = \eta^2 V^2 \sum_{b \in \cN_{ak}} \tilde P_{tk}^2 / P_{tb}
\leq \eta^2 K^2 V^2/ \gamma = \eta KV \leq 1$,
where in the last inequality we used the definition of $\eta$ and assumed that $n \geq K \log(K/\delta)$.
To make sure that the regret bound holds even for smaller values of $n$, 
we require $C_G \ge K\sqrt{\log(eK)}$ so that when $n<K^2 \log(K/\delta)$, the regret bound is trivial. For assumption \textit{(c)}, we have
\begin{align*}
\E_{t-1}[\hat Z_{tka}^2] 
&= \E_{t-1}\left[\left(\frac{\tilde P_{tk} v^{ak}(A_t, \Phi_t)}{P_{tA_t}}\right)^2\right] 
\leq V^2 \tilde P_{tk}^2 \E_{t-1}\left[\frac{\sind{\cN_{ak}}(A_t)}{P_{tA_t}^2}\right] 
= V^2 \sum_{b \in \cN_{ak}} \frac{\tilde P_{tk}^2}{P_{tb}} 
= \frac{\beta_{tka}}{\eta}\,.
\end{align*}
Finally \textit{(d)} is satisfied by the definition of $v^{ak}$ and the fact that $P_t \in \ri(\cP_{K-1})$.
The result of \cref{thm:pm:exp3p} shows that with probability at least $1 - (K+1)\delta$,
\begin{align*}
R_{nk} 
&\leq \frac{3\log(1/\delta)}{\eta} + 5\sum_{t=1}^n \sum_{a \in \cN_k \cap \cA} Q_{tka} \beta_{tka} + \eta \sum_{t=1}^n \sum_{a \in \cN_k \cap \cA} Q_{tka} \hat Z_{tka}^2\,. 
\end{align*}

\paragraph{Step 2: Aggregating the local regret}
Using the result from the previous step in combination with a union bound over $k \in \cA$ we have that with probability at least $1 - K(K+1) \delta$,
\begin{align}
\sum_{k\in \cA} R_{nk} 
&\leq 
    \frac{3K \log(1/\delta)}{\eta} 
 + 5\sum_{t=1}^n \sum_{k \in \cA} \sum_{a \in \cN_k \cap \cA} Q_{tka} \beta_{tka} 
    + \eta \sum_{t=1}^n \sum_{k \in \cA} \sum_{a \in \cN_a \cap \cA} Q_{tka}\hat Z_{tka}^2 \,.
\label{eq:pm:split}
\end{align}
For bounding the second term 
we use the definition of $\beta_{tka}$ from \eqref{eq:pm:estimators} and write
\begin{align*}
\sum_{a \in \cN_k \cap \cA} Q_{tka} \beta_{tka}
= \eta V^2 \sum_{a \in \cN_k \cap \cA} Q_{tka} \sum_{b \in \cN_{ak}} \frac{\tilde P_{tk}^2}{P_{tb}}
= \eta V^2 \tilde P_{tk} \sum_{a \in \cN_k \cap \cA} Q_{tka} \sum_{b \in \cN_{ak}} \frac{\tilde P_{tk}}{P_{tb}}\,.
\end{align*}
The sum over $b\in \cN_{ak}$ is split into two, separating duplicates of $k$ and the rest:
\begin{align*}
\sum_{a \in \cN_k \cap \cA} Q_{tka} \sum_{b \in \cN_{ak}} \frac{\tilde P_{tk}}{P_{tb}}
& =\sum_{a \in \cN_k \cap \cA} Q_{tka} \sum_{b : \ell_b = \ell_k} \frac{\tilde P_{tk}}{P_{tb}} 
  + \sum_{a \in \cN_k \cap \cA} Q_{tka} \sum_{b \in \cN_{ak} : \ell_b \neq \ell_k} \frac{\tilde P_{tk}}{P_{tb}} \\
&= \sum_{b : \ell_b = \ell_k} \frac{\tilde P_{tk}}{P_{tb}}
+ \sum_{a \in \cN_k \cap \cA} \sum_{b \in \cN_{ak} : \ell_b \neq \ell_k} \frac{Q_{tka} \tilde P_{tk}}{P_{tb}} \\
&\leq 4  K  \left( \sum_{b: \ell_b = \ell_k} 1 + \sum_{a \in \cN_k \cap \cA} \sum_{b \in \cN_{ak} : \ell_b \neq \ell_k} 1\right) 
\leq 4 K^2\,,
\end{align*}
where the first equality used that $\sum_a Q_{tka}=1$,
 the second to last inequality follows using parts (c) and (e) of \cref{lem:pm:redistribute},
and the last inequality uses the reasoning above.
%The final inequality follows since for distinct $a,a'$ neighbors of $k$ the intersection $\cN_{ak} \cap \cN_{a'k}$ only contains $k$ and its duplicates (see figure).
Summing over all rounds and $k \in \cA$ yields
\begin{align*}
5\sum_{t=1}^n \sum_{k \in \cA} \sum_{a \in \cN_k \cap \cA} Q_{tka}\beta_{tka}
\leq 20\eta nK^2 V^2\,.
\end{align*}
For the last term in \cref{eq:pm:split} we use the definition of $\hat Z_{tka}$ and \cref{lem:pm:redistribute}.(c) to show that
\begin{align*}
\eta \sum_{t=1}^n\sum_{k \in \cA} \sum_{a \in \cN_k \cap \cA} Q_{tka} \hat Z_{tka}^2
&= \eta \sum_{t=1}^n\sum_{k \in \cA} \sum_{a \in \cN_k \cap \cA} \frac{Q_{tka} \tilde P_{tk}^2 v^{ak}(A_t, \Phi_t)^2}{P_{tA_t}^2} \\
&\leq \eta V^2 \sum_{t=1}^n \frac{1}{P_{tA_t}} \sum_{k \in \cA} \tilde P_{tk} \sum_{a \in \cN_k \cap \cA} \frac{Q_{tka} \tilde P_{tk} \sind{\cN_{ak}}(A_t)}{P_{tA_t}} \\
&\leq 4\eta K V^2 \sum_{t=1}^n \frac{1}{P_{tA_t}}\,.
\end{align*}
Now, from \cref{lem:pm:redistribute} (d), $\gamma/K\, (1/P_{ta}) \leq 1$ for all $a$, and in particular, holds for $a=A_t$. Furthermore, $\E_{t-1}[1/P_{tA_t}] = K$ and $\E_{t-1}[1/P_{tA_t}^2] = \sum_a 1/P_{ta} \leq K^2 / \gamma$. 
By the result in \cref{lem:pm:conc} it holds that with probability at least $1 - \delta$ that
\begin{align*}
\sum_{t=1}^n \frac{1}{P_{tA_t}} \leq 2nK + \frac{K \log(1/\delta)}{\gamma}\,. 
\end{align*}
Another union bound shows that with probability at least $1 - (1 + K(K + 1))\delta$,
\begin{align*}
\sum_{k\in \cA} R_{nk} &\leq \frac{3K \log(1/\delta)}{\eta} + 28\eta nV^2 K^2 + 4VK \log(1/\delta)\,.
\end{align*}
The result follows from the definition of $\eta$, \cref{lem:pm:transform} and the definition of $R_{nk}$.
\end{proof}

\section{Proof of Theorem~\ref{thm:pm}}\label{app:thm:pm}
Before the proof we need a simple lemma showing that if actions $a$ and $b$ are not weak neighbours, then
the regret of either $a$ or $b$ grows linearly in $t$. 

\begin{lemma}\label{lem:pm:sep}
There exists a game-dependent constant $\epsilon_G > 0$ such that:
\begin{enumerate}
\item[(a)] If $a$ and $b$ are not weak neighbours, then $\inf_{u \in \cP_{E-1}} \ip{\ell_a + \ell_b - 2\ell_{a^*_u}, u} \geq \epsilon_G$. 
\item[(b)] If $u \in \cP_{E-1}$ and $M_u = \{a : \ip{\ell_a - \ell_{a^*_u}, u} < \epsilon_G\}$, then $|M_u| \leq \Kloc$.
\end{enumerate}
\end{lemma}

\begin{proof}
For (a) let $c \in [K]$ be arbitrary. Since $C_a \cap C_b = \emptyset$ it follows that $\ip{\ell_a + \ell_b - 2\ell_c, u} > 0$ for all $u \in C_c$.
By compactness of $C_c$ and the continuity of the inner product in $u$ we conclude that $\inf_{u \in C_c} \ip{\ell_a + \ell_b - 2\ell_c, u} > 0$.
Taking the minimum over all $c$ shows that $2\epsilon_G = \inf_{u \in \cP_{E-1}} \ip{\ell_a + \ell_b - 2\ell_{a^*_u}, u} > 0$. 
For (b), let $a,b \in M_u$. Then $\ip{\ell_a + \ell_b - 2\ell_{a^*_u}, u} < 2\epsilon_G$, which by (a) means that $a$ and $b$ are weak neighbours. 
Therefore all actions in $M_u$ are weak neighbours of each other so $|M_u| \leq \Kloc$.
\end{proof}

\newcommand{\fail}{\text{\textsc{Fail}}}

The next lemma uses the concentration of the loss estimators to show that with high probability the distribution $\tilde P_t$ calculated by \textsc{RelExp3} assigns negligible probability
to actions that are either not neighbours of $B_t$ or for which the loss is large relative to the optimal action.

\begin{lemma}\label{lem:bern}
Let $Z_{ta} = \ip{\ell_a - \ell_{B_t}, u_t}$ and $L_{ta} = \sum_{s=1}^t Z_{sa}$. Then there exists an event $\fail$ with $\Prob{\fail} \leq 1/n$ and function $g:\N \to [0,\infty)$
such that if $\fail$ does not hold, then
\begin{enumerate}
\item[(a)] $\tilde P_{ta} \leq \exp(-\eta_t g(t))$ for all $a$ that are not neighbours of $B_t$.
\item[(b)] $\tilde P_{ta} \leq \exp(-\eta_t g(t))$ for all $a$ with $\ip{\ell_a - \ell_{a^*_t}, \bar u_t} \geq \epsilon_G$.
\item[(c)] There exist constants $c_1, c_2 \geq 0$ depending on $G = (\cL, \Phi)$ and the choice of $\epsilon$ in the definition of $\alpha_t$
such that for all $t \geq c_1 \log^{c_2}(n)$ it holds that $g(t) \geq \frac{1}{2} \epsilon_G t$. 
\end{enumerate}
\end{lemma}

\begin{proof}
Define random variable $\phi_t = \max_{a,b} |\sum_{s=1}^t (\hat Z_{sb} - Z_{sb} + Z_{sa} - \hat Z_{sa})|$.
Given an arbitrary pair of arms $(a, b)$, from the triangle inequality we have
\begin{align*}
\left|\hat L_{ta} - \hat L_{tb}\right|
&= \left|\sum_{s=1}^t \left(\hat Z_{sa} - \hat Z_{sb}\right)\right| 
\geq \left|\sum_{s=1}^t \left(Z_{sa} - Z_{sb}\right)\right| - \left| \sum_{s=1}^t \left(\hat Z_{sa} - Z_{sa} + Z_{sb} - \hat Z_{sb}\right) \right|\\
&= \left|L_{ta} - L_{tb}\right| - \left| \sum_{s=1}^t \left(\hat Z_{sa} - Z_{sa} + Z_{sb} - \hat Z_{sb}\right)\right|
\ge \left|L_{ta} - L_{tb}\right| - \phi_t\,.
\end{align*}
The quantity $\phi_t$ is bounded with high probability via a union bound over all pairs of arms and a martingale version of Bernstein's bound \citep{Fre75}, which shows there exists
a game-dependent constant $C_G > 0$ such that
\begin{align*}
\mathbb{P}\big(\underbrace{\text{exists } t \in [n] : \phi_t \geq C_G t^{3/4+\epsilon/2} \log^{\frac{1}{2}}(n)}_{\fail}\big)  \leq \frac{1}{n}\,.
\end{align*}
Choose $g(t) = \max\{0, (t-1) \epsilon_G - C_G t^{3/4+\epsilon/2} \log^{\frac{1}{2}}(n)\}$, which clearly satisfies the condition in \textit{(c)}.
First suppose that $a$ is not a weak neighbour of $B_{t+1}$, which by the definition of $B_{t+1}$, $\phi_t$ and \cref{lem:pm:sep} ensures that
\begin{align*}
\hat L_{ta} - \hat L_{tB_{t+1}}
\geq \hat L_{ta} + \hat L_{tB_{t+1}} - 2\hat L_{ta^*_t}
\geq L_{ta} + L_{tB_{t+1}} - 2 L_{ta^*_t} - 2 \phi_t
\geq t\epsilon_G - 2\phi_t\,.
\end{align*}
On the other hand if $\ip{\ell_a - \ell_{a^*_t}, \bar u_t} \geq \epsilon_G$, then
\begin{align*}
\hat L_{ta} - \hat L_{tB_{t+1}}
\geq \hat L_{ta} - \hat L_{ta^*_t} 
\geq L_{ta} - L_{ta^*_t} - 2\phi_t
\geq t\epsilon_G - 2\phi_t \,.
\end{align*}
The result follows from the fact that $\tilde P_{ta} \leq \exp(\eta_t (\hat L_{t-1,B_t} - \hat L_{t-1,a}))$ for any action.
\end{proof}

\begin{proof}[Proof of \cref{thm:pm}]
Choose $\epsilon \in (0,1/2)$ and
\begin{align*}
\eta_t = \min\set{\frac{1}{4K V},\, \frac{1}{2\Vloc} \sqrt{\frac{\log(K)}{2t \Kloc}}} \qquad \text{and} \qquad \alpha_t = \min\set{\frac{1}{4K},\, t^{-1/2-\epsilon}} \,.
\end{align*}
First note that the choices of $\eta_t$ and $\alpha_t$ ensures that $\norm{\gamma_t}_1 \leq 1/2$ and so $P_t$ is indeed a probability distribution and $P_{ta} \geq \tilde P_{ta} / 2$
for all $t$ and $a$. Let $N_t$ be the set of weak neighbours of $B_t$. Since $\E_{t-1}[\hat Z_{ta}] = Z_{ta}$ we have for $p = e_{a^*_n}$ that
\begin{align*}
\E[R_n] 
&= \E\left[\sum_{t=1}^n Z_{tA_t} - Z_{ta^*_n}\right] 
= \E\left[\sum_{t=1}^n \ip{\tilde P_t - p, Z_t}\right] + \E\left[\sum_{t=1}^n \ip{\gamma_t - \Vert\gamma_t\Vert_1 \tilde P_t, Z_t}\right] \\
&\leq \E\Bigg[\underbrace{\sum_{t=1}^n \ip{\tilde P_t - p, \hat Z_t}}_{\tilde R_n}\Bigg] + 2\E\left[\sum_{t=1}^n \norm{\gamma_t}_1 \right]\,,
\end{align*}
where in the inequality we used Cauchy-Schwartz and the fact that $\norm{Z_t}_\infty \leq 1$.
The analysis of exponential weights given in Theorem 2.3 of the book by \cite{Ces06} yields
\begin{align*}
\E[\tilde R_n] 
&\leq \log(K) \left(\frac{2}{\eta_n} + \frac{1}{\eta_1}\right) + \sum_{t=1}^n \E\Bigg[
\underbrace{\sum_{a=1}^K \tilde P_{ta} \hat Z_{ta} + \frac{1}{\eta_t} \log\left(\sum_{a=1}^K \tilde P_{ta} \exp\left(-\eta_t \hat Z_{ta}\right)\right)}_{\textrm{(A)}_t}\Bigg]\,. 
\end{align*}
Note that we have stopped the proof before the application of Hoeffding's lemma, which is not appropriate for bandits due to the large range of the loss estimates.
Suppose that $a \in M_t$, then for any $b \in S^{aB_t} \subseteq S_t$ we have $P_{tb} \geq \gamma_{tb} \geq \eta_t V^{aB_t}$, which means that
\begin{align*}
\eta_t |\hat Z_{ta}| = \frac{\eta_t v^{aB_t}(A_t, \Phi_t)}{P_{tA_t}} \leq \frac{\eta_t V^{aB_t} \sind{S^{aB_t}}(A_t)}{P_{tA_t}} \leq 1\,.
\end{align*}
Then using $\exp(x) \leq 1 + x + x^2$ for $x \leq 1$ leads to
\begin{align*}
\tilde P_{ta} \exp\left(-\eta_t \hat Z_{ta}\right) \leq \tilde P_{ta} - \eta_t \tilde P_{ta}\hat Z_{ta} + \eta_t^2 \tilde P_{ta} \hat Z_{ta}^2\,.
\end{align*}
On the other hand, if $a \notin M_t$ then by the definitions of $M_t$, $\hat Z_{ta}$ and $P_t$,
\begin{align*}
\tilde P_{ta} \exp\left(-\eta_t \hat Z_{ta}\right) \leq \tilde P_{ta} \exp\left(\eta_t|\hat Z_{ta}|\right) \leq \tilde P_{ta} \exp\left(\frac{\eta_t V^{aB_t}}{\alpha_t}\right) \leq \frac{\eta_t}{t}\,,
\end{align*}
which by the fact that $x \leq 1 + x \leq \exp(x)$ for all $x$ also implies that $\tilde P_{ta} \hat Z_{ta} \leq 1/t$.
Using $\log(1+x) \leq x$,
\begin{align*}
\textrm{(A)}_t 
&= \sum_{a=1}^K \tilde P_{ta} \hat Z_{ta} + \frac{1}{\eta_t} \log\left(\sum_{a=1}^K \tilde P_{ta} \exp\left(-\eta_t \hat Z_{ta}\right)\right) \\
&\leq \sum_{a=1}^K \tilde P_{ta} \hat Z_{ta} + \frac{1}{\eta_t} \log\left(\frac{\eta_t K}{t} + 1 - \eta_t \sum_{a \in M_t} \tilde P_{ta} \hat Z_{ta} + \eta_t^2 \sum_{a \in M_t} \tilde P_{ta} \hat Z_{ta}^2\right) \\ 
&\leq \frac{2K}{t} + \eta_t \sum_{a \in M_t} \tilde P_{ta} \hat Z_{ta}^2\,,
\end{align*}
Next we bound the conditional second moment of $\hat Z_{ta}$. If $a$ and $B_t$ are weak neighbours, then
\begin{align*}
\tilde P_{ta} \E_{t-1}[\hat Z_{ta}^2] 
&= \tilde P_{ta} \E_{t-1}\left[\frac{v^{aB_t}(A_t, \Phi_t)^2}{P_{tA_t}^2}\right] \\ 
&\leq \tilde P_{ta}\Vloc^2 \E_{t-1}\left[\frac{\one{A_t \in \{a, B_t\}}}{P_{tA_t}^2}\right] 
\leq 2\Vloc^2 + o(1)\,,
\end{align*}
where in the last line we used the fact that $P_{tA_t} \geq P_{ta}$ whenever $A_t \in \{a, B_t\}$ and $\E_{t-1}[\one{A_t \in \{a, B_t\}}] = 2$ and 
$P_{tA_t} \geq (1 - \norm{\gamma_t}_1) \tilde P_{tA_t} = \tilde P_{tA_t}(1 - o(1))$.
On the other hand, if $a$ and $B_t$ are not weak neighbours, then $\E_{t-1}[\hat Z_{ta}^2] \leq K^2V^2 / \alpha_t$ and so
\begin{align}
\sum_{t=1}^n \E\left[\textrm{(A)}_t\right]
&\leq \sum_{t=1}^n \frac{2K}{t} +\sum_{t=1}^n \eta_t \E\left[\sum_{a \in M_t \cap N_t} \tilde P_{ta} \hat Z_{ta}^2\right]
+ \sum_{t=1}^n \eta_t \E\left[\sum_{a \in M_t \cap N_t^c} \tilde P_{ta} \hat Z_{ta}^2\right] \nonumber \\
&\leq 2\Vloc^2 \sum_{t=1}^n \eta_t \E\left[|M_t|\right] + K^2V^2 \sum_{t=1}^n \frac{\eta_t}{\alpha_t} \E\left[\sum_{a \in N_t^c} \tilde P_{ta}\right] + o(\sqrt{n})\,. \label{eq:A1}
\end{align}
The second sum is bounded using parts \textit{(a)} and \textit{(c)} of \cref{lem:bern}, which shows that
\begin{align}
\sum_{t=1}^n \frac{\eta_t}{\alpha_t} \E\left[\sum_{a \in N_t^c} \tilde P_{ta}\right] 
\leq \Prob{\fail} \sum_{t=1}^n \frac{\eta_t}{\alpha_t} + \sum_{t=1}^n \frac{\eta_t}{\alpha_t} \sum_{a \in N_t^c} \exp\left(-\eta_t g(t)\right)
= o(\sqrt{n})\,. \label{eq:A1-a}
\end{align}
Suppose that $\fail$ does not hold and define $t_0$ by
\begin{align*}
t_0 = \min\set{t : \text{for all } s \geq t,\, \exp\left(\frac{\eta_s V}{\alpha_s} - \eta_s g(s)\right) \leq \frac{\eta_s}{s}}\,,
\end{align*}
which by part \textit{(c)} of \cref{lem:bern} and rearrangement satisfies $t_0 = O(\operatorname{polylog}(n))$.
The definition of $t_0$ ensures that if $t \geq t_0$ and $a$ is an action with $\ip{\ell_a - \ell_{a^*_{t-1}}, \bar u_{t-1}} > \epsilon_G$, then 
\begin{align*}
\tilde P_{ta} \exp\left(\frac{\eta_t V^{aB_t}}{\alpha_t}\right)
\leq \exp\left(\frac{\eta_t V^{aB_t}}{\alpha_t} - \eta_t g(t)\right) 
\leq \exp\left(\frac{\eta_t V}{\alpha_t} - \eta_t g(t)\right) \leq \frac{\eta_t}{t}
\end{align*}
and so $a \notin M_t$. Therefore when $\fail$ does not hold and $t \geq t_0$,
\begin{align*}
M_t \subseteq \{a : \ip{\ell_a - \ell_{a^*_{t-1}}, \bar u_{t-1}} \leq \epsilon_G\}\,.
\end{align*}
But by \cref{lem:pm:sep} the number of arms in this set is at most $\Kloc$ and so in this case $|M_t| \leq \Kloc$.
Since $|M_t| \leq K$ regardless of $t$ or the failure event,
\begin{align*}
\sum_{t=1}^n \eta_t \E[|M_t|]
&\leq \Prob{\fail} \sum_{t=1}^n \eta_t K + \sum_{t=1}^{t_0} \eta_t K + \Kloc \sum_{t=t_0+1}^n \eta_t 
= \Kloc \sum_{t=1}^n \eta_t + o(\sqrt{n})\,.
\end{align*}
Combining the above display with \cref{eq:A1,eq:A1-a} shows that
\begin{align*}
\sum_{t=1}^n \E[\textrm{(A)}_t] = 2 \Vloc^2 \Kloc \sum_{t=1}^n \eta_t + o(\sqrt{n})\,. 
\end{align*}
Next we bound the sum of the expectations of $\norm{\gamma_t}_1$. To begin notice that if $M_t$ contains only neighbours of $B_t$, then $S_t = M_t$ and $\max_{a \in M_t} V^{aB_t} \leq \Vloc$.
The definitions of $\gamma_t$ and $\alpha_t = o(\sqrt{1/t})$ means that $\norm{\gamma_t}_1 = \eta_t |S_t| \max_{a \in M_t} V^{aB_t} + o(\sqrt{1/t})$ and so the same argument as above shows that
\begin{align*}
2 \E\left[\sum_{t=1}^n \norm{\gamma_t}_1\right] = 2\Vloc \Kloc \sum_{t=1}^n \eta_t + o(\sqrt{n})\,.
\end{align*}
Putting the pieces together and using the fact that $\sum_{t=1}^n \sqrt{1/t} \leq 2\sqrt{n}$,
\begin{align*}
\limsup_{n\to\infty} \frac{\E[R_n]}{\sqrt{n}} 
&\leq \limsup_{n\to\infty} \frac{1}{\sqrt{n}} \left(\frac{2\log(K)}{\eta_n} + 2 \Kloc \Vloc(\Vloc + 1) \sum_{t=1}^n \eta_t\right) \\
&= 8 \sqrt{2 \Vloc(1 + \Vloc)\Kloc n \log(K)}\,.
\end{align*}
The result is completed by recalling that $v^{ab}(\cdot)$ were chosen so that $\Vloc \leq 1 + F$.
\end{proof}

%%%%%%%%%%%%%%%%%%%%%%%%%%%%%%%%%%%%%%%%%%%%%%%%%%%%%%%%%%%
% LOWER BOUNDS
%%%%%%%%%%%%%%%%%%%%%%%%%%%%%%%%%%%%%%%%%%%%%%%%%%%%%%%%%%%
\section{Lower Bounds for Hard Games}\label{sec:pm:lower}
In this section we prove a $\Omega(n^{2/3})$ lower bound on the minimax regret in hard partial monitoring games.
Like for bandits, by Yao's minimax principle \citep{Yao77}, 
the lower bounds are most easily proven using a stochastic adversary.
In stochastic partial monitoring we assume that $u_1,\ldots,u_n$ are chosen independently at random from the same distribution. To emphasise the randomness
we switch to capital letters. Given a partial monitoring problem $G = (\cL, \Phi)$ and a probability vector $u \in \cP_{E-1}$ the stochastic partial monitoring environment associated with $u$ 
samples a sequence of independently and identically distribution random variables $U_1,\ldots,U_n$ with $U_t \in \{e_1,\ldots,e_E\}$ with $\Prob{U_t = e_i} = u_i$.
In each round $t$ a policy chooses action $A_t$ and receives feedback $\Phi_t = \Phi(A_t, U_t)$.
The regret is
\begin{align*}
R_n(u,\pi,G) = \max_{a \in [K]} \E\left[\sum_{t=1}^n \ip{\ell_{A_t} - \ell_a, U_t}\right] = \max_{a \in [K]} \E\left[\sum_{t=1}^n \ip{\ell_{A_t} - \ell_a, u}\right]\,.
\end{align*}
The mentioned minimax principle implies that $R_n^*(G) \ge \inf_\pi\sup_u  R_n(u,\pi,G)$.
Hence, in what follows, we lower bound  $\sup_u  R_n(u,\pi,G)$ for fixed $\pi$. 

Given $u,v \in \cP_{E-1}$, let $\KL(u,v)$ be the relative entropy between categorical distributions with parameters $u$ and $v$ respectively: 
\begin{align}
\KL(u, v) = \sum_{i=1}^K u_i \log\left(\frac{u_i}{v_i}\right) \leq \sum_{i=1}^K \frac{(u_i - v_i)^2}{v_i}\,, 
\label{eq:pm:kl-cat}
\end{align}
where the second inequality follows from the fact that for measures $P \ll Q$ we have $\KL(P, Q) \leq \chi^2(P,Q)$.
We need one more simple result that is frequently used in lower bound proofs. Given measures $P$ and $Q$ on the same probability space, Lemma~2.6 in the book by \cite{Tsy08} says that for any event $A$,
\begin{align}
P(A) + Q(A^c) \geq \frac{1}{2} \exp\left(-\KL(P, Q)\right)\,.
\label{eq:hppinsker}
\end{align}

\begin{theorem}\label{thm:pm:hard-lower}
Let $G = (\cL, \Phi)$ be a globally observable partial monitoring problem with that is not locally observable. Then there exists a constant $c_G > 0$ such that
$R^*_n(G) \geq c_G n^{2/3}$.
\end{theorem}

\begin{proof}
The proof involves several steps. Roughly, we need to define two alternative stochastic partial monitoring problems. We then show these environments
are hard to distinguish without playing an action associated with a large loss. Finally we balance the cost of distinguishing the environments against the linear
cost of playing randomly. 

Fix a policy $\pi$ and a partial monitoring game $G$ with the required properties.
For $u \in \cP_{E-1}$ let $\bbP_u$ denote the measure on sequences of outcomes $(A_1, \Phi_1,\ldots,A_n, \Phi_n)$ induced by the interaction
of a fixed policy and the stochastic partial monitoring problem determined by $u$ and $G$
and denote by $\E_{u}$ the corresponding expectation.
Note that $R_n(\pi, u, G) =\max_a \E_{u}\left[\sum_{t=1}^n \ip{\ell_{A_t} - \ell_a, u}\right]$.

\paragraph{Step 1: Defining the alternatives}
Let $a, b$ be a pair neighbouring actions that are not locally observable. Then by definition $C_a \cap C_b$ is a polytope of dimension $E-2$.
Let $u$ be the centroid of $C_a \cap C_b$ and
\begin{align}
\epsilon = \min_{c \notin \cN_{ab}} \ip{\ell_c - \ell_a, u}\,.
\label{eq:pm:epsilon}
\end{align}
The value of $\epsilon$ is well-defined, since by global observability of $G$, but nonlocal observability of $(a,b)$ there must exist some action $c \notin \cN_{ab}$.
Furthermore, since $c \notin \cN_{ab}$ it follows that $\epsilon > 0$.  We also have $u \in \ri(\cP_{E-1})$.
We now define two stochastic partial monitoring problems.
Since $(a,b)$ are not locally observable, there is no function $v:[K] \times [F] \to \R$ such that for all $i \in [E]$,
\begin{align}
\sum_{c \in \cN_{ab}} v(c, \Phi_{ci}) = \ell_{ai} - \ell_{bi}\,.
\label{eq:pm:h2}
\end{align}
To facilitate the next step we rewrite this using a linear structure.
For action $c \in [K]$
let $S_c \in \{0,1\}^{F \times E}$ be the matrix with $(S_c)_{f i} = \one{\Phi(c, i) = f}$, which is chosen
so that $S_c e_i = e_{\Phi_{ci}}$.
Define the linear map $S: \R^E \to \R^{|\cN_{ab}|F}$ by
\begin{align*}
S = \begin{pmatrix} 
S_a \\
S_b \\
\vdots \\
S_c
\end{pmatrix}\,,
\end{align*}
which is the matrix formed by stacking the matrices $\{S_c : c \in \cN_{ab}\}$.
Then there exists a $v$ satisfying \cref{eq:pm:h2} if and only if there exists a vector $w \in \R^{|\cN_{ab}|F}$ such that 
\begin{align*}
\ell_a - \ell_b = w^\top S\,.
\end{align*}
In other words, actions $(a,b)$ are locally observable if and only if $\ell_a - \ell_b \in \img(S^\top)$.
Since we have assumed that $(a,b)$ are not locally observable, it follows that $\ell_a - \ell_b \notin \img(S^\top)$.
Let $z \in \img(S^\top)$ and $w \in \ker(S)$ be such that $\ell_a - \ell_b = z + w$, which is possible since $\img(S^\top) \oplus \ker(S) = \R^E$.
Since $\ell_a - \ell_b \notin \img(S^\top)$ it holds that
$w \neq 0$ and $\ip{w, \ell_a - \ell_b} = \ip{w, z + w} = \ip{w,w} \neq 0$. Finally let $v = w / \ip{w, \ell_a - \ell_b}$. 
It follows that $Sv = 0$ and $\ip{v, \ell_a - \ell_b} = 1$.\footnote{The minor error in \cite{BFPRS14} appears in their definition of $v$, which is in the kernel of a \textit{different} $S$ constructed
by stacking just $S_a$ and $S_b$ and not the degenerate/duplicate actions in between.}
Let $\Delta > 0$ be some small constant to be tuned subsequently and
define $u_a = u - \Delta v$ and $u_b = u + \Delta v$ so that
\begin{align*}
\ip{\ell_b - \ell_a, u_a} = \Delta \qquad \text{and} \qquad \ip{\ell_a - \ell_b, u_b} = \Delta\,.
\end{align*}
We note that if $\Delta$ is sufficiently small, then $u_a \in C_a \cap \ri(\cP_{E-1})$ and $u_b \in C_b\cap \ri(\cP_{E-1})$.
This means that action $a$ is optimal if the environment plays $u_a$ on average and $b$ is optimal if the environment plays $u_b$ on average and that $u_a$ and $u_b$ are in the relative interior of the $(E-1)$-simplex (see \cref{fig:pm:lower}).

\paragraph{Step 2: Calculating the relative entropy}
Given action $c$ and $w \in \cP_{E-1}$ let $\bbP_{cw}$ be the distribution on the feedback observed by the learner when playing action $c$ in stochastic partial monitoring environment determined by $w$.
That is $\bbP_{cw}(f) = \bbP_w(\Phi_t = f|A_t = c) = (S_c w)_f$. 
Let $T_c(n)$ be the number of times action $c$ is played over all $n$ rounds. The chain rule for relative entropy shows that
\begin{align}
\KL(\bbP_{u_a}, \bbP_{u_b}) = \sum_{c \in [K]} \E_{u_a}[T_c(n)] \KL(\bbP_{cu_a}, \bbP_{cu_b})\,.
\label{eq:pm:kl-decomp}
\end{align}
By definition of $u_a$ and $u_b$ we have $S_c u_a = S_c u_b$ for all $c \in \cN_{ab}$. 
Therefore $\bbP_{cu_a} = \bbP_{cu_b}$ and so $\KL(\bbP_{cu_a}, \bbP_{cu_b}) = 0$ for all $c \in \cN_{ab}$.
On the other hand, if $c \notin \cN_{ab}$, then thanks to $u_a,u_b,u\in \ri(\cP_{E-1})$ and \cref{eq:pm:kl-cat}, 
\begin{align*}
\KL(\bbP_{cu_a}, \bbP_{cu_b}) \leq \KL(u_a, u_b) \leq \sum_{i=1}^E \frac{(u_{ai} - u_{bi})^2}{u_{bi}} = 4\Delta^2 \sum_{i=1}^K \frac{v_i^2}{u_i - \Delta v_i}
\leq C_u \Delta^2\,,
\end{align*}
where $C_u$ is a suitably large constant and we assume that $\Delta$ is chosen sufficiently small that $u_i - \Delta v_i \geq u_i/2$. Therefore
\begin{align}
\KL(\bbP_{u_a}, \bbP_{u_b}) \leq C_u  \E_{u_a}[\tilde T(n)] \Delta^2\,,
\label{eq:pm:kl}
\end{align}
where $\tilde T(n)$ is the number of times an arm not in $\cN_{ab}$ is played:
\begin{align*}
\tilde T(n) = \sum_{c \notin \cN_{ab}} T_c(n)\,.
\end{align*}

\paragraph{Step 3: Comparing the regret}

By \cref{eq:pm:epsilon} and the Cauchy-Schwartz inequality for $c \notin \cN_{ab}$ we have
$\ip{\ell_c - \ell_a, u_a} 
= \epsilon + \ip{\ell_c - \ell_a, \Delta v} 
\geq \epsilon - 2\Delta \norm{v}_\infty$ and 
$\ip{\ell_c - \ell_b, u_b} \geq \epsilon - 2\Delta \norm{v}_\infty$.
By \cref{lem:pm:degenerate},
for each action $c \in \cN_{ab}$ there exists an $\alpha \in [0,1]$ such that $\ell_c = \alpha \ell_a + (1 - \alpha) \ell_b$. Therefore
\begin{align}
\ip{\ell_c - \ell_a, u_a} + \ip{\ell_c - \ell_b, u_b} = (1 - \alpha) \ip{\ell_b - \ell_a, u_a} + \alpha \ip{\ell_a - \ell_b, u_b} = \Delta\,,
\label{eq:pm:deltasplit}
\end{align}
which means that $\max(\ip{\ell_c - \ell_a, u_a}, \ip{\ell_c - \ell_b, u_b}) \geq \Delta / 2$.
Define $\bar T(n)$ as the number of times some arm in $\cN_{ab}$ is played that is at least $\Delta/2$ suboptimal in $u_a$:
\begin{align*}
\bar T(n) = \sum_{c \in \cN_{ab}} \one{\ip{\ell_c - \ell_a, u_a} \geq \frac{\Delta}{2}} T_c(n)\,.
\end{align*}
Assume that $\Delta$ is chosen sufficiently small so that $2\Delta \norm{v}_\infty \leq \epsilon / 2$. Then
\begin{align*}
&R_n(\pi, u_a, G) + R_n(\pi, u_b, G) 
= \E_{u_a}\left[\sum_{c \in [K]} T_c(n) \ip{\ell_c - \ell_a, u_a}\right] + \E_{u_b}\left[\sum_{c \in [K]}T_c(n) \ip{\ell_c - \ell_b, u_b}\right] \\
&\qquad\quad\geq \frac{\epsilon}{2} \E_{u_a}\left[\tilde T(n)\right] + \frac{n\Delta}{4} \left(\bbP_{u_a}(\bar T(n) \geq n/2) + \bbP_{u_b}(\bar T(n) < n/2)\right) \\
&\qquad\quad\geq \frac{\epsilon}{2} \E_{u_a}\left[\tilde T(n)\right] + \frac{n\Delta}{8} \exp\left(-\KL(\bbP_{u_a}, \bbP_{u_b})\right) \\
&\qquad\quad\geq \frac{\epsilon}{2} \E_{u_a}\left[\tilde T(n)\right] + \frac{n\Delta}{8} \exp\left(-C_u \Delta^2 \E_{u_a}\left[\tilde T(n)\right]\right)\,,
\end{align*}
In the above display we used \eqref{eq:pm:deltasplit} and
\begin{align*}
\sum_c T_c(n) \one{ \ip{\ell_c-\ell_b,u_b}>\Delta/2} 
= \sum_c T_c(n) \one{ \ip{\ell_c-\ell_a,u_a}\le \Delta/2 } = n-\bar T(n)\,,
\end{align*}
where the second inequality follows from the high probability version of Pinsker's inequality \cref{eq:hppinsker} and the third from \cref{eq:pm:kl}.
The bound is completed by a simple case analysis.
If $\E_{u_a}[\tilde T(n)] > n^{2/3}$, the result holds for any value of $\Delta$.
Otherwise choosing $\Delta = (c/n)^{1/3}$ for appropriate positive game-dependent constant $c$ establishes the bound.
\end{proof}

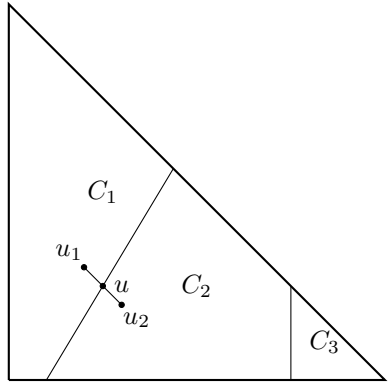
\begin{figure}[h]
\centering
\begin{tikzpicture}
\draw[thick] (0,0) -- (5,0) -- (0,5) -- (0,0);
\draw (0.5,0) -- (2.1875,2.8124);
\node at (1.25,2.5) {$C_1$};
\node at (2.5,1.25) {$C_2$};
\draw[fill=black] (1.25,1.25) circle (1pt);
\node at (1.5,1.25) {$u$};
\draw (1.25,1.25) -- (1,1.5);
\draw (1.25,1.25) -- (1.5,1);
\draw[fill=black] (1,1.5) circle (1pt);
\draw[fill=black] (1.5,1) circle (1pt);
\node at (0.8,1.7) {$u_1$};
\node at (1.7,0.8) {$u_2$};
\draw (3.75,0) -- (3.75,1.25);
\node at (4.2,0.5) {$C_3$};
\end{tikzpicture}
\caption{Lower bound construction for hard partial monitoring problems}\label{fig:pm:lower}
\end{figure}

%%%%%%%%%%%%%%%%%%%%%%%%%%%%%%%%%%%%%%%%%%%%%%%%%%%%%
% MORE LOWER BOUNDS
%%%%%%%%%%%%%%%%%%%%%%%%%%%%%%%%%%%%%%%%%%%%%%%%%%%%%
\section{Lower Bound for Theorem~\ref{thm:pm}}
\label{sec:pm:lower-more}
We consider the following game with $K = 2$ and $E = 2F - 2$ and $G = (\cL, \Phi)$ given by
\begin{align*}
\cL = \begin{pmatrix}
1 & 0 & 1 & 0 & \cdots & 1 & 0 \\
0 & 1 & 0 & 1 & \cdots & 0 & 1 
\end{pmatrix}\,, \quad
\Phi = \begin{pmatrix} 
1 & 2 & 2 & 3 & 3 & 4 & \cdots & F - 1 & F - 1 & F \\
1 & 1 & 2 & 2 & 3 & 3 & \cdots & F - 2 & F - 1 & F - 1
\end{pmatrix}\,.
\end{align*}

\begin{theorem}\label{thm:pm:F-lower}
For $n$ sufficiently large the minimax regret of $G$ is at least $R^*_n(G) \geq \frac{F-1}{45} \sqrt{n}$.
\end{theorem}

\begin{proof}
Let $\Delta = \sqrt{1/(17n)}$ and $u \in \cP_{E-1}$ be constants to be tuned subsequently and
$u' \in \cP_{E-1}$ by $u'_i = u_i + 2(-1)^i \Delta$. Using the notation of the previous section we have
\begin{align*}
\KL(\bbP_{1u}, \bbP_{1u'}) \leq \chi^2(\bbP_{1u}, \bbP_{2u'}) = 4\Delta^2 \left(\frac{1}{u'_1} + \frac{1}{u'_E}\right) 
\qquad \text{and} \qquad
\KL(P_{2u}, P_{2u'}) = 0\,.
\end{align*}
Let $p = 1/2 - (E-2)\Delta / 4$ and
choose $u_1 = p + \Delta$ and $u_E = p - \Delta$ and for $i \in \{2,\ldots, E-1\}$ let $u_i = \Delta(1 - (-1)^i)$. For sufficiently large horizon, $\Delta$ is small enough so that
$\KL(\bbP_{1u}, \bbP_{1u'}) \leq 17 \Delta^2$ and 
\begin{align*}
\KL(\bbP_u, \bbP_{u'}) = \E_u[T_1(n)] \KL(\bbP_{1u}, \bbP_{1u'}) \leq 17 n \Delta^2\,.
\end{align*}
Using the fact that $R_n(u) = 2\Delta (F-1) \E_u[T_1(n)]$ and $R_n(u') = 2\Delta (F-1) \E_{u'}[T_2(n)]$ 
leads to
\begin{align*}
R_n(u) + R_n(u') \geq \frac{n (F-1)\Delta}{2} \exp\left(-n 17\Delta^2\right) = \frac{(F-1) \sqrt{\frac{n}{17}}}{2e}\,.
\end{align*}
The result follows since $\max(a, b) \geq (a+b)/2$ and by naive simplification.
\end{proof}

%%%%%%%%%%%%%%%%%%%%%%%%%%%%%%%%%%%%%%%%%%%%%%%%%%%%%%%%%%%%
% GENERIC BOUND FOR EXP3P
%%%%%%%%%%%%%%%%%%%%%%%%%%%%%%%%%%%%%%%%%%%%%%%%%%%%%%%%%%%%
\section{Generic Bound for Exponential Weights}

The proof of \cref{thm:pm:easy} depends on a generic regret analysis for a variant of the \textsc{Exp3.P} bandit algorithm by \cite{ACFS95}. 
The main difference is that loss estimators are assumed to be god-given and satisfy certain
properties, rather than being explicitly defined as biased importance-weighted estimators. Nothing here would startle an expert, but we do not know where an equivalent result is written in the literature.
Let $(\Omega, \cF, (\cF_t)_{t=0}^n, \bbP)$ be a filtered probability space and abbreviate $\E_t[\cdot] = \E[\cdot|\cF_t]$.
To reduce clutter we assume for the remainder that $t$ ranges in $[n]$ and $a \in [K]$. Recall that a sequence of random elements $(X_t)$ is called adapted if $X_t$ is $\cF_t$-measurable for all $t$, while $(X_t)$ is called predictable
if $X_t$ is $\cF_{t-1}$-measurable for all $t$.
Let $(Z_t)$ and $(\tilde Z_t)$ be sequences of random elements in $\R^K$.
Given nonempty $\cA \subseteq [K]$ and positive constant $\eta$ define the probability vector $Q_t \in \cP_{K-1}$ by
\begin{align*}
Q_{ta} = \frac{\sind{\cA}(a) \exp\left(-\eta \sum_{s=1}^{t-1} \tilde Z_{sa}\right)}{\sum_{b \in \cA} \exp\left(-\eta \sum_{s=1}^{t-1} \tilde Z_{sb}\right)} \,.
\end{align*}

\begin{theorem}\label{thm:pm:exp3p}
Assume that the $\R^K$-valued process $(Z_t)_t$ is predictable,
the $\R^K$-valued process $(\tilde Z_t )_t$ is adapted and that $\tilde Z_t = \hat Z_t - \beta_t$, 
where $(\hat Z_t)_t$ is adapted and $(\beta_t)_t$ is predictable.
Assume the following hold for all $a \in \cA$:
\begin{align*}
&\text{\textit{(a)}}\quad \eta |\hat Z_{ta}| \leq 1\,, \qquad \qquad
&&\text{\textit{(b)}}\quad \eta \beta_{ta} \leq 1\,, \qquad \qquad \\
&\text{\textit{(c)}}\quad \eta \E_{t-1}[\hat Z_{ta}^2] \leq \beta_{ta} \,\, \text{almost surely}\,, \qquad \qquad
&&\text{\textit{(d)}}\quad \E_{t-1}[\hat Z_{ta}] = Z_{ta} \,\, \text{almost surely} \,.
\end{align*}
Let $A^* = \argmin_{a \in \cA} \sum_{t=1}^n Z_{ta}$. Then, for any $0\le \delta\le 1/(K+1)$, 
with probability at least $1 - (K+1)\delta$,
\begin{align*}
\sum_{t=1}^n \sum_{a=1}^K Q_{ta} (Z_{ta} - Z_{tA^*}) \leq 
\frac{3 \log(1/\delta)}{\eta} + \eta \sum_{t=1}^n \sum_{a \in \cA} Q_{ta} \hat Z_{ta}^2 + 5 \sum_{t=1}^n \sum_{a \in \cA} Q_{ta} \beta_{ta}\,.
\end{align*}
\end{theorem}

\begin{proof}
We proceed in five steps.

\paragraph{Step 1: Decomposition}
\begin{align*}
&\sum_{t=1}^n \sum_{a \in \cA} Q_{ta} (Z_{ta} - Z_{tA^*}) \\ 
&\qquad= \underbrace{\sum_{t=1}^n \sum_{a\in\cA} Q_{ta} (\tilde Z_{ta} - \tilde Z_{tA^*})}_{\textrm{(A)}} 
  + \underbrace{\sum_{t=1}^n \sum_{a \in \cA} Q_{ta} (Z_{ta} - \tilde Z_{ta})}_{\textrm{(B)}} 
  + \underbrace{\sum_{t=1}^n (\tilde Z_{tA^*} - Z_{tA^*})}_{\textrm{(C)}}\,.
\end{align*}

\paragraph{Step 2: Bounding (A)}
By assumption \textit{(c)} we have $\beta_{ta} \geq 0$, which by assumption \textit{(a)} means that
$\eta \tilde Z_{ta} \leq \eta \hat Z_{ta} \leq \eta |\hat Z_{ta}| \leq 1$ for all $a \in \cA$. 
Then the standard mirror descent analysis with negentropy regularisation \citep{Haz16} shows that (A) is bounded by
\begin{align*}
\textrm{(A)} 
&\leq \frac{\log(K)}{\eta} + \eta \sum_{t=1}^n \sum_{a \in \cA} Q_{ta} \tilde Z_{ta}^2 \\
&= \frac{\log(K)}{\eta} + \eta \sum_{t=1}^n \sum_{a \in \cA} Q_{ta} (\hat Z_{ta}^2 + \beta_{ta}^2) - \frac{2\eta}{1 - \gamma} \sum_{t=1}^n \sum_{a \in \cA} Q_{ta} \hat Z_{ta} \beta_{ta} \\
&\leq \frac{\log(K)}{\eta} + \eta \sum_{t=1}^n \sum_{a\in \cA} Q_{ta} \hat Z_{ta}^2 + 3 \sum_{t=1}^n \sum_{a \in \cA} Q_{ta} \beta_{ta}\,,
\end{align*}
where in the last two line we used the assumptions that $\eta \beta_{ta} \leq 1$ and $\eta |\hat Z_{ta}| \leq 1$.

\paragraph{Step 3: Bounding (B)}
For (B) we have
\begin{align*}
\textrm{(B)} 
&= \sum_{t=1}^n \sum_{a\in \cA} Q_{ta} (Z_{ta} - \tilde Z_{ta})
= \sum_{t=1}^n \sum_{a \in \cA} Q_{ta} (Z_{ta} - \hat Z_{ta} + \beta_{ta})\,.
\end{align*}
We prepare to use \cref{lem:pm:conc}. By assumptions \textit{(c)} and \textit{(d)} respectively we have 
$\eta \E_{t-1}[\hat Z_{ta}^2] \leq \beta_{ta}$ and $\E_{t-1}[\hat Z_{ta}] = Z_{ta}$.
By Jensen's inequality,
\begin{align*}
\eta \E_{t-1}\left[\left(\sum_{a \in \cA} Q_{ta} (Z_{ta} - \hat Z_{ta})\right)^2\right]
&\leq \eta \sum_{a \in \cA} Q_{ta} \E_{t-1}[\hat Z_{ta}^2]
\leq \sum_{a \in \cA} Q_{ta} \beta_{ta}\,.
\end{align*}
Therefore by \cref{lem:pm:conc}, with probability at least $1 - \delta$ 
\begin{align*}
\textrm{(B)} \leq 2\sum_{t=1}^n \sum_{a \in \cA} Q_{ta} \beta_{ta} + \frac{\log(1/\delta)}{\eta}\,. 
\end{align*}

\paragraph{Step 4: Bounding (C)}
For (C) we have
\begin{align*}
\textrm{(C)} 
= \sum_{t=1}^n (\tilde Z_{tA^*} - Z_{tA^*}) 
= \sum_{t=1}^n \left(\hat Z_{tA^*} - Z_{tA^*} - \beta_{tA^*}\right)\,.
\end{align*}
Because $A^*$ is random we cannot directly apply \cref{lem:pm:conc}, but need a union bound over all actions. Let $a \in \cA$ be fixed.
Then by \cref{lem:pm:conc} and the assumption that $\eta|\hat Z_{ta}| \leq 1$ and $\E_{t-1}[\hat Z_{ta}] = Z_{ta}$ and $\eta \E_{t-1}[\hat Z_{ta}^2] \leq \beta_{ta}$, 
with probability at least $1 - \delta$.
\begin{align*}
\sum_{t=1}^n \left(\hat Z_{ta} - Z_{ta} - \beta_{ta} \right) \leq \frac{\log(1/\delta)}{\eta}\,.
\end{align*}
Therefore by a union bound we have with probability at most $1 - K\delta$,
\begin{align*}
\textrm{(C)} \leq \frac{\log(1/\delta)}{\eta}\,.
\end{align*}

\paragraph{Step 5: Putting it together}
Combining the bounds on (A), (B) and (C) in the last three steps with the decomposition in the first step shows that with probability at least $1 - (K+1)\delta$,
\begin{align*}
R_n \leq \frac{3 \log(1/\delta)}{\eta} + \eta \sum_{t=1}^n \sum_{a \in \cA} Q_{ta} \hat Z_{ta}^2 + 5 \sum_{t=1}^n \sum_{a \in \cA} Q_{ta} \beta_{ta}\,.
\end{align*}
where we used the assumption that $\delta \leq 1/K$.
\end{proof}

%%%%%%%%%%%%%%%%%%%%%%%%%%%%%%%%%%%%%%%%%%%%%%%%%%%%%
% BOUNDING V
%%%%%%%%%%%%%%%%%%%%%%%%%%%%%%%%%%%%%%%%%%%%%%%%%%%%%
\section{Bounding the Norm of the Estimators}\label{sec:pm:v}

\begin{lemma}\label{lem:F}
Let $a$ and $b$ be pairwise observable and $\Delta = \ell_a - \ell_b \in [-1,1]^E$, then there exists a function $v : \{a, b\} \times [F] \to \R$ such that:
\begin{enumerate}
\item[(a)] $\norm{v}_\infty \leq 1+F$.
\item[(b)] $v(a, \Phi_{ai}) + v(b, \Phi_{bi}) = \Delta_i$ for all $i \in [E]$.
\end{enumerate}
\end{lemma}

\begin{proof}
By the definition of pairwise observable there exists a function $v : \{a, b\} \times [F] \to \R$ such that $v(a, \Phi_{ai}) + v(b, \Phi_{bi}) = \Delta_i$ for all $i \in [E]$.
Define bipartite graph with vertices $\cV = \{a, b\} \times [F]$ and edges $\cE = \{((a, \Phi_{ai}), (b, \Phi_{bi})) : i \in [E]\}$.
Assume without loss of generality this graph is fully connected. If not then apply the following procedure to each connected component.
For any $f, f' \in [F]$ let $(a, f_1),(b,f_2),(a,f_3),\ldots,(a,f_k)$ be a loop free path with $f_1 = f$ and $f_k = f'$ and
for $j \in [k-1]$ let $i_j \in [E]$ correspond to the edge connecting $(\cdot, f_j)$ and $(\cdot, f_{j+1})$. Then
\begin{align*}
\left|v(a, f) - v(a, f')\right| = \left|\sum_{j=1}^{k-1} (-1)^{j-1} \Delta_{i_j}\right| \leq 2F\,.
\end{align*}
We may assume that $\max_{f \in [F]} |v(a, f)| \leq \frac{1}{2}\max_{f,f' \in [F]} |v(a, f) - v(a, f')|$, which is always possible by translating $v(a, \cdot)$ 
by a constant $\alpha$ and $v(b, \cdot)$ by $-\alpha$.
Finally for each $f \in [F]$ there exists an $f' \in [F]$ and $i \in [E]$ such that $v(b, f) + v(a, f') = \Delta_i$, which ensures that $|v(b,f)| \leq |v(a,f')| + |\Delta_i| \leq F + 1$.
\end{proof}

%%%%%%%%%%%%%%%%%%%%%%%%%%%%%%%%%%%%%%%%%%%%%%%%%%%%%%%%%
% TECHNICAL RESULTS
%%%%%%%%%%%%%%%%%%%%%%%%%%%%%%%%%%%%%%%%%%%%%%%%%%%%%%%%%
\section{Concentration}\label{sec:technical}

\begin{lemma}\label{lem:pm:conc}
Let $X_1,X_2,\ldots,X_n$ be a sequence of random variables adapted to filtration $(\cF_t)_t$ and let $\E_t[\cdot] = \E[\cdot|\cF_t]$ and $\mu_t = \E_{t-1}[X_t]$.
Suppose that $\eta > 0$ satisfies $\eta X_t \leq 1$ almost surely.
Then
\begin{align*}
\Prob{\sum_{t=1}^n (X_t - \mu_t) \geq \eta \sum_{t=1}^n \E_{t-1}[X_t^2]\, + \frac{1}{\eta} \log\left(\frac{1}{\delta}\right)} \leq \delta\,.
\end{align*}
\end{lemma}

\begin{proof}
Let $\sigma^2_t = \E_{t-1}[X_t^2]$.
By Chernoff's method we have
\begin{align*}
\Prob{\sum_{t=1}^n \left(X_t - \mu_t - \eta \sigma_t^2\right) \geq \frac{\log(1/\delta)}{\eta}}
&= \Prob{\exp\left(\eta \sum_{t=1}^n \left(X_t - \mu_t - \eta \sigma^2_t\right)\right) \geq \frac{1}{\delta}} \\
&\leq \delta \E\left[\exp\left(\eta \sum_{t=1}^n \left(X_t - \mu_t - \eta \sigma^2_t\right)\right)\right]\,.
\end{align*}
The result is completed by showing the term inside the expectation is a supermartingale.
For this, we have
\begin{align*}
\E_{t-1}\left[\exp\left(\eta \left(X_t - \mu_t - \eta \sigma^2_t\right)\right)\right]
&= \exp\left(-\eta \mu_t - \eta^2 \sigma^2_t\right) \E_{t-1}\left[\exp\left(\eta X_t\right)\right] \\
&\leq \exp\left(-\eta \mu_t - \eta^2 \sigma^2_t\right) \left(1 + \eta \mu_t + \eta^2 \sigma_t^2\right) 
\leq 1\,,
\end{align*}
where we used the inequalities $\exp(x) \leq 1 + x + x^2$ for $x \leq 1$ and $1 + x \leq \exp(x)$ for all $x \in \R$.
Chaining the above inequality completes the proof.
\end{proof}

\newcommand{\erfc}{\operatorname{erfc}}

\begin{proof}[Proof of Lemma~\ref{lem:dom}]
Let $\Lambda$ be the smallest value such that $X \leq Y + \sqrt{2a \log(b/\Lambda)}$, which is almost surely positive.
Taking expectations of both sides shows that
\begin{align*}
\E[X] \leq \E[Y] + \sqrt{2a} \E[\sqrt{\log(b/\Lambda)}]
\end{align*}
The second expectation is bounded by
\begin{align*}
\E[\sqrt{\log(b/\Lambda)}] 
&= \int^\infty_0 \Prob{\sqrt{\log(b/\Lambda)} \geq x} dx \\
&\leq \inf_{y > 0} y + \int^\infty_y \Prob{\sqrt{\log(b/\Lambda)} \geq x} dx \\
&\leq \inf_{y > 0} y + \int^\infty_y \Prob{\Lambda \leq b \exp(-x^2)} dx \\
&\leq \inf_{y > 0} y + \int^\infty_y b \exp(-x^2) dx \\
&= \inf_{y > 0} y + \frac{b\sqrt{\pi}}{2} \erfc(y) \\
&= \sqrt{\log(b)} + \frac{b \sqrt{\pi}}{2} \erfc(\sqrt{\log(b)}) \\
&\leq \sqrt{1 + \log(b)}\,. \qedhere
\end{align*}
\end{proof}

%%%%%%%%%%%%%%%%%%%%%%%%%%%%%%%%%%%%%%%%%%%%%%%%%%%%%
% GALLERY
%%%%%%%%%%%%%%%%%%%%%%%%%%%%%%%%%%%%%%%%%%%%%%%%%%%%%
\section{Gallery}\label{sec:pm:gallery}

\begin{exhibit}\label{exmpl:pm:no-estimates}
In the following game the learner cannot estimate the actual losses, but the loss differences can be calculated from the feedback directly.
\begin{align*}
\cL = \begin{pmatrix}
1 & 1/2 & 1/2 & 0 \\
1/2 & 1 & 0 & 1/2
\end{pmatrix}\,, \qquad
\Phi = \begin{pmatrix}
1 & 2 & 1 & 2 \\
1 & 2 & 1 & 2
\end{pmatrix}\,.
\end{align*}
\end{exhibit}

\begin{exhibit}\label{exmpl:pm:spam}
A useful way to think about the cell decomposition is to assume that $\cL$ has positive entries and 
consider the intersection of the hypograph of concave function $f(u) = \min_a \ip{\ell_a, u}$ with domain $u = \cP_{E-1}$ and the epigraph of $\cP_{E-1}$.
To illustrate the idea let $G = (\cL, \Phi)$ be the variant of the spam game where $c = 1/3$, which is defined by 
\begin{align*}
\cL = \begin{pmatrix}
1 & 0 \\
0 & 1 \\
\frac{1}{3} & \frac{1}{3}
\end{pmatrix}\,,
\qquad
\Phi = \begin{pmatrix}
1 & 1 \\
1 & 1 \\
1 & 2
\end{pmatrix}\,.
\end{align*}
In this case $\cP_{E-1} = \cP_1$ is $1$-dimensional, which means the intersection of the epigraph of $\cP_{E-1}$ and the hypograph of $f$ is 2-dimensional and is shown in the left figure below.
The intersection is itself a polytope and the faces (1-dimensional in this case) pointing upwards correspond to cells of nondegenerate actions. If $c$ is increased to $1/2$, then the third
action becomes degenerate, which is observable from the right-hand figure below by noting that the dimension of its intersection with the polytope is now zero. Increasing $c$ any further
would make this action dominated.
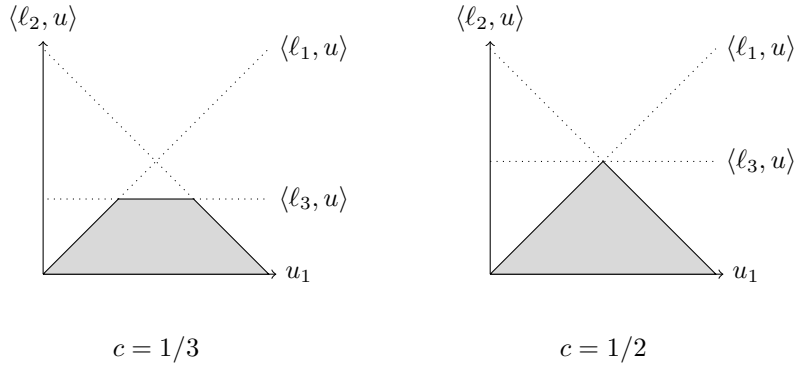
\begin{figure}[H]
\centering
\begin{tikzpicture}
\draw[->] (0,0) -- (3.1,0);
\draw[->] (0,0) -- (0,3.1);
\draw[fill=gray!30!white] (0,0) -- (1,1) -- (2,1) -- (3,0) -- (0,0);
\draw[dotted] (1,1) -- (3,3);
\draw[dotted] (1,1) -- (0,1);
\draw[dotted] (2,1) -- (3,1);
\draw[dotted] (2,1) -- (0,3);
\node at (3.6,1) {$\ip{\ell_3, u}$};
\node at (3.6,3) {$\ip{\ell_1, u}$};
\node at (0,3.4) {$\ip{\ell_2, u}$};
\node at (3.4,0) {$u_1$};
\node at (1.5,-1) {$c = 1/3$};
\end{tikzpicture}
\hspace{1cm}
\begin{tikzpicture}
\draw[->] (0,0) -- (3.1,0);
\draw[->] (0,0) -- (0,3.1);
\draw[fill=gray!30!white] (0,0) -- (1.5,1.5) -- (3,0) -- (0,0);
\draw[dotted] (1.5,1.5) -- (3,3);
\draw[dotted] (1.5,1.5) -- (0,3);
\draw[dotted] (0,1.5) -- (3,1.5);
\node at (3.6,1.5) {$\ip{\ell_3, u}$};
\node at (3.6,3) {$\ip{\ell_1, u}$};
\node at (0,3.4) {$\ip{\ell_2, u}$};
\node at (3.4,0) {$u_1$};
\node at (1.5,-1) {$c = 1/2$};
\end{tikzpicture}

\caption{Alternative view of cell decomposition for spam game with $c = 1/3$ and $c = 1/2$.} 
\end{figure}
\end{exhibit}

\begin{exhibit}\label{exmpl:pm:not-local}
The following game demonstrates that not all locally observable games are point-locally observable. \\
\begin{minipage}{7cm}
\begin{align*}
\cL = \begin{pmatrix}
0 & 1 & 1 \\
1 & 0 & 1 \\
1/2 & 1/2 & 1/2
\end{pmatrix}\,,
\qquad
\Phi = \begin{pmatrix}
1 & 1 & 1 \\
1 & 1 & 1 \\
1 & 2 & 3
\end{pmatrix}\,.
\end{align*}
\end{minipage}
\hspace{1cm}
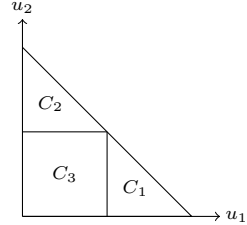
\begin{minipage}{5cm}
\centering
\begin{tikzpicture}[font=\scriptsize,scale=0.75]
\draw[->] (0,0) -- (3.5,0);
\draw[->] (0,0) -- (0,3.5);
\draw (0,0) -- (3,0) -- (0,3);
\draw (0,1.5) -- (1.5,1.5);
\draw (1.5,0) -- (1.5,1.5);
\node at (0.75,0.75) {$C_3$};
\node at (0.5, 2) {$C_2$};
\node at (2, 0.5) {$C_1$};
\node at (3.8, 0) {$u_1$};
\node at (0, 3.7) {$u_2$};
\end{tikzpicture}
\end{minipage} \\
The cell decomposition for this game is shown on above-right. Notice that $(1,2)$ are not neighbours, but are weak neighbours.
And yet $(1,2)$ are not pairwise observable. Therefore the game is not point-locally observable. On the other hand, both sets of neighbours $(1,2)$ and $(1,3)$ 
are locally observable.
\end{exhibit}

\begin{exhibit}\label{exmpl:pm:eliminates}
This game produces the cell decomposition depicted at the start of Section~\ref{sec:pm:nw2}. 
The only neighbours are $(2,3)$ and $(1,3)$, which are locally observable. Therefore the game is locally observable. 
Actions 4,5 and 6 are degenerate. \textsc{NeighbourhoodWatch2} will only play actions $1$, $2$, $3$ and $4$ with actions $5$ and $6$ ruled
out because their cells are not equal to the intersection of any neighbours cells. Notice that $\ell_4 = \ell_2 / 2 + \ell_3/2$ is a convex combination of $\ell_2$ and $\ell_3$. \\
\begin{minipage}{7cm}
\begin{align*}
\cL = \begin{pmatrix}
0 & 1 & 1 \\
1 & 0 & 1 \\
1/2 & 1/2 & 1/2 \\
3/4 & 1/4 & 3/4 \\
1 & 1/2 & 1/2 \\
1 & 1/4 & 3/4
\end{pmatrix}\,,
\qquad
\Phi = \begin{pmatrix}
1 & 1 & 1 \\
1 & 1 & 1 \\
1 & 2 & 3 \\
1 & 1 & 1 \\
1 & 1 & 1 
\end{pmatrix}\,.
\end{align*}
\end{minipage}
\hspace{1cm}
\begin{minipage}{6cm}
\begin{tikzpicture}[font=\scriptsize]
\begin{scope}
\clip (0,0) -- (3,0) -- (0,3) -- (0,0);
\draw[draw=none,fill=gray!10!white] (0,0) -- (0,1.5) -- (1.5,1.5) -- (1.5,0) -- (0,0);
\draw[draw=none,fill=gray!30!white] (1.5,0) -- (1.5,1.5) -- (3,0) -- (1.5,0);
\draw[draw=none,fill=gray!30!white] (0,1.5) -- (0,3) -- (1.5,1.5) -- (0,1.5);
\draw[densely dotted,ultra thick] (0,1.5) -- (1.5,1.5);
\end{scope}
\draw[densely dotted,ultra thick] (0,1.5) -- (0,0);
\draw[fill=black] (0,1.5) circle (2pt);
\draw[->] (0,0) -- (3.3,0);
\draw[->] (0,0) -- (0,3.3);
\node at (0.75,0.75) {$C_3$};
\node at (0.375, 1.875) {$C_2$};
\node at (1.875,0.375) {$C_1$};
\node[anchor=west] at (3.4,0) {$u_1$};
\node[anchor=south] at (0,3.4) {$u_2$};
\node[inner sep=0pt] (C4) at (1.5,2.3) {$C_4$};
\draw[thin,->,shorten >=3pt] (C4) -- (0.75,1.5);
\node[inner sep=0pt] (C5) at (-1,0.75) {$C_5$};
\draw[thin,->,shorten >=3pt] (C5) -- (0, 0.75);
\node[inner sep=0pt] (C6) at (-1,1.5) {$C_6$};
\draw[thin,->,shorten >=3pt] (C6) -- (0, 1.5);
\end{tikzpicture}
\end{minipage}
\end{exhibit}

\fi

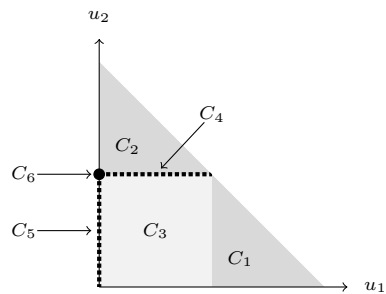
\end{document}